\documentclass[]{interact}
\usepackage{epstopdf}
\usepackage[caption=false]{subfig}
\usepackage[numbers,sort&compress]{natbib}
\bibpunct[, ]{[}{]}{,}{n}{,}{,}
\renewcommand\bibfont{\fontsize{10}{12}\selectfont}

\usepackage{amsthm,amsmath,amsfonts,amssymb}
\usepackage{graphicx}
\usepackage{multirow}
\usepackage[title]{appendix}
\usepackage{xcolor}
\usepackage{textcomp}
\usepackage{manyfoot}
\usepackage{booktabs}
\usepackage{algorithm}
\usepackage{algorithmicx}
\usepackage{algpseudocode}
\usepackage{listings}
\usepackage{enumitem}
\usepackage{mathtools}
\usepackage{comment}
\usepackage{slantsc}

\newcommand{\Dom}{\mathrm{Dom}\,}

\newcommand{\E}{\mathbb{E}}
\newcommand{\R}{\mathbb{R}}

\newcommand{\dt}{dt}

\theoremstyle{plain}
\newtheorem{theorem}{Theorem}[section]
\newtheorem*{theorem*}{Theorem}
\newtheorem{lemma}[theorem]{Lemma}
\newtheorem{corollary}[theorem]{Corollary}

\theoremstyle{definition}

\theoremstyle{remark}
\newtheorem{remark}[theorem]{Remark}

\numberwithin{equation}{section}
\usepackage[colorlinks,citecolor=blue,urlcolor=blue]{hyperref}

\begin{document}

\articletype{RESEARCH ARTICLE}

\title{A Malliavin calculus approach to score functions in diffusion generative models}

\author{
\name{Ehsan Mirafzali\textsuperscript{a}\thanks{CONTACT Ehsan Mirafzali. Email: smirafza@ucsc.edu}, 
Frank Proske\textsuperscript{b}, 
Utkarsh Gupta\textsuperscript{a}, \\
Daniele Venturi\textsuperscript{c}, and 
Razvan Marinescu\textsuperscript{a}}
\affil{\textsuperscript{a}Department of Computer Science, University of California Santa Cruz; \\
\textsuperscript{b}Department of Mathematics, University of Oslo; \\
\textsuperscript{c}Department of Applied Mathematics, University of California Santa Cruz}
}

\maketitle

\begin{abstract}
Score-based diffusion generative models have recently emerged as a standard tool for modelling complex data distributions. These models aim at learning the score function, which defines a map from a known probability distribution to the target data distribution via deterministic or stochastic differential equations (SDEs). The score function is typically estimated from data using a variety of approximation techniques, such as denoising or sliced score matching, Hyv\"arien's method, or Schr\"odinger bridges. In this paper, we derive an exact, closed-form expression for the score function for a broad class of nonlinear diffusion generative models. Our approach combines modern stochastic analysis tools such as Malliavin derivatives and their adjoint operators (Skorokhod integrals or Malliavin Divergence) with a new Bismut-type formula. The resulting expression for the score function can be written entirely in terms of the first and second variation processes, with all Malliavin derivatives systematically eliminated, thereby enhancing its practical applicability. The theoretical framework presented in this work offers a principled foundation for advancing score estimation methods in generative modelling, enabling the design of new sampling algorithms for complex probability distributions. Our results can be extended to broader classes of stochastic differential equations, opening new directions for the development of score-based diffusion generative models.
\end{abstract}


\section{Introduction}
\label{sec:intro}

Score-based diffusion generative models \cite{song2021scorebased,pmlr-v37-sohl-dickstein15} have recently emerged as a powerful tool for modelling complex data distributions in a variety of applications ranging from finance \cite{10.1145/3604237.3626876} to image synthesis \cite{NEURIPS2021_49ad23d1, 9878449}. These models aim to estimate the so-called score function \(\nabla_y \log p_t(y)\), which governs the transformation of a simple reference distribution into the target distribution via a time-reversed nonlinear diffusion process.

Historically, score-based generative models have been built upon diffusion processes described by linear It\^o stochastic differential equations, i.e., SDEs with drift and diffusion coefficients that are linear in the state variables. Such a linearity assumption ensures analytical tractability and comes with other desirable properties. For example, the Fokker--Planck equation~\cite{https://doi.org/10.1002/andp.19143480507, planck1917satz} associated with the diffusion process admits a Gaussian stationary distribution and Gaussian transition densities, which have been used in the development of techniques such as Denoising Diffusion Probabilistic Models (DDPM) \cite{NEURIPS2020_4c5bcfec}.
While the linearity assumption in diffusion processes has been instrumental in developing the first generation of score-based generative models, it also represents an overly restrictive simplification. Many real-world systems exhibit nonlinear dynamics, where both the drift and diffusion coefficients depend nonlinearly on the state variables. Extending diffusion models to such settings poses substantial challenges. 
Specifically, unlike linear SDEs, nonlinear SDEs typically lack closed-form expressions for their transition probability densities, making the 
computation of the score function \( \nabla_y \log p_t(y) \) challenging. 

%

The purpose of this paper is to bridge this gap by establishing a new rigorous link between score-based diffusion models and Malliavin calculus, a stochastic calculus of variations introduced by Paul Malliavin in the 1970s \cite{malliavin:78:stochastic, 10.5555/262387}. Originally devised to investigate hypoelliptic operators and Stochastic Partial Differential Equations (SPDEs), Malliavin calculus provides a rigorous framework for analysing the smoothness of SDE solutions and their densities. By leveraging Malliavin derivatives and a novel Bismut-type formula we derive in this paper, which is distinct from the classical Bismut--Elworthy--Li formula \cite{alma990005308880204808,ELWORTHY1994252, Elworthy_1982} tailored to heat kernels, we derive an explicit (closed-form) representation of the score function for general nonlinear diffusion processes. The resulting formula is expressed entirely in terms of first and second variation processes \cite{kunita1997stochastic, bismut1981martingales}, with all Malliavin derivatives systematically eliminated.
In particular, our framework generalises the state‐independent results of \cite{mirafzali2025malliavincalculusscorebaseddiffusion} to fully nonlinear SDEs with state‐dependent diffusion coefficients. For the numerical experiments of this framework, see the same work \cite{mirafzali2025malliavincalculusscorebaseddiffusion}. Recently, complementary advances have appeared. First, an infinite-dimensional extension develops score-based diffusion modelling on separable Hilbert spaces using Malliavin tools and an operator-theoretic formulation \cite{mirafzali2025scorebaseddiffusionmodelsinfinite}. A related direction employs Malliavin–Gamma calculus to treat random fields, providing an abstract Hilbertian setting for score-based diffusion generative models \cite{greco2025malliavingammacalculusapproachscore}. In parallel, a divergence–kernel representation is derived for scores of random dynamical systems, accommodating multiplicative noise and yielding additive-noise and short-time specialisations \cite{ni2025divergencekernelmethodscoresrandom}. Finally, a path-space integration-by-parts approach addresses conditioning and diffusion bridges, producing a generalised Tweedie-type score identity and numerically robust training for singular rewards \cite{pidstrigach2025conditioning}.

Beyond establishing the link with Malliavin calculus, our contribution is to move from gradient identities for expectations to score formulae that are usable in diffusion models. The central obstacle is that the classical Bismut--Elworthy--Li framework returns \( \nabla_x \mathbb{E}[f(X_T^x)] \) via Skorokhod (anticipative) integrals, whereas score-based methods require the density-level quantity \( \nabla_x \log p_t(x) \). We address this by first replacing the anticipative objects with adapted, pathwise quantities, time integrals taken along the forward SDE, so that the resulting estimators are numerically implementable in both the linear regime (recovering the Fokker--Planck score) and the nonlinear setting with state-independent diffusion. Taken together, these steps convert the abstract BEL-type representations into practical, computationally robust algorithms for score-based diffusion modelling.

The implications of this new theoretical framework are important. 
First, Malliavin calculus provides a tool for overcoming the analytical challenges posed by nonlinear diffusion processes. Second, its rigorous foundations 
enable the development and analysis of new classes of generative  
models beyond linear diffusion.  From a practical standpoint, 
this can potentially yield more accurate generative models for 
applications such as high-fidelity image synthesis, fluid dynamics \cite{WickMalliavin2013}, and protein folding 
simulations in molecular dynamics \cite{Abramson2024}.



This paper is organised as follows. In Section \ref{sec:problem-statement} we state our main theorem expressing the score function  \(\nabla_y \log p_t(y)\) in terms of first and second variation processes via a Bismut-type formula. Unlike the Bismut-Elworthy-Li formula, our formula directly targets the score function, leveraging the first and second variation processes to yield a practical closed-form expression.  In Section \ref{sec:malliavin} we briefly review basic notions of Malliavin calculus and prove various theorems on covering vector fields. In Section \ref{sec:variation} we discuss first and second variation processes in detail. In Section \ref{sec:proofmain}  we prove our main theorem, i.e., Theorem \ref{thm:main}. The proof includes the full derivation of the score function for nonlinear SDEs, along with several useful lemmas that assist in the process. In Section \ref{sec:simplifiedformula} we provide a simplified expression of the score function for nonlinear SDEs with state-independent diffusion coefficients. 

\section{Problem statement and main result}
\label{sec:problem-statement}
Consider the \(m\)-dimensional stochastic differential equation
\begin{equation}
dX_t = b(t, X_t)\,dt + \sigma(t, X_t)\,dB_t, 
\qquad X_0 = x, 
\qquad 0 \le t \le T,
\label{sde}
\end{equation}
where \(X_t \in \mathbb{R}^m\), \(B_t \in \mathbb{R}^d\) is a standard Brownian motion, 
\(b : [0,T] \times \mathbb{R}^m \to \mathbb{R}^m\) is the drift coefficient, and 
\(\sigma : [0,T] \times \mathbb{R}^m \to \mathbb{R}^{m \times d}\) is the diffusion coefficient, 
with \(\sigma^{l}(t,X_t)\) denoting its \(l\)-th column for \(l = 1,\ldots,d\). 
The initial condition \(x \in \mathbb{R}^m\) is deterministic, and \(T > 0\) is a fixed terminal time. 
We assume that \(b\) and \(\sigma\) are \(\mathcal{C}^{2}\) functions with bounded derivatives, which ensures the existence of a unique strong solution to~\eqref{sde}. 


In the framework of score-based diffusion models, the SDE \eqref{sde} represents the so-called {\em forward diffusion process}, which gradually transforms the initial data distribution into some other distribution over the time interval \([0, T]\). To enable generative sampling, we define the corresponding {\em reverse-time diffusion process}, which runs backwards in time from \(T\) to \(0\). Based on the general structure of reverse-time SDEs, this process is given by
\begin{align}
dX_t = \left\{ b(t, X_t) - \nabla \cdot \left[\sigma(t, X_t) \sigma(t, X_t)^\top\right] + \sigma(t, X_t) \nabla_{X_t} \log p_t(X_t) \right\} \, dt + \sigma(t, X_t) \, d\bar{B}_t, 
\label{sde1}
\end{align}
for all $0 \leq t \leq T$, where \(\nabla \cdot [\sigma(t, X_t) \sigma(t, X_t)^\top]\) denotes the divergence of the matrix \(\sigma(t, X_t) \sigma(t, X_t)^\top\) (a correction term arising from the Itô interpretation of the SDE), \(\nabla_{X_t} \log p_t(X_t)\) is the score function, and \(d\bar{B}_t\) represents the increment of a reverse-time Brownian motion.
The reverse-time SDE effectively reverses the forward diffusion by leveraging 
the score function to guide the process from a noise distribution back to the 
original data distribution.
\noindent
Our goal is to compute the score function
\[
\nabla_y \log p_t(y) = \frac{\nabla_y p_t(y)}{p_t(y)}, \quad \partial_{y_k} \log p_t(y) = \frac{\partial_{y_k} p_t(y)}{p_t(y)}, \quad k = 1, \ldots, m,
\]
which is essential for implementing the reverse diffusion process in generative tasks. In particular, we are interested in \(\nabla_y \log p_T(y)\)\footnote{For generative tasks, one requires the time-dependent score \(\nabla_y\log p_t(y)\) for each intermediate time \(t\), which is obtained by evaluating the integrals discussed in Theorem \ref{thm:main} at every such \(t\).}. To achieve this goal, we leverage Malliavin calculus and a new Bismut-type formula we develop in this paper. As we shall see hereafter, this allows us to express the score in terms of the first variation process \(Y_t = \partial X_t/\partial x\) and the second variation process \(Z_t = \partial^2 X_t/\partial x^2\), which capture the first- and the second-order sensitivity of the solution \(X_t\) relative to changes in the initial condition \(x\).   
Our main result is the following theorem.


\begin{theorem}[Skorokhod integral representation theorem for the score function]
\label{thm:main}
Let \(X_t\) be the solution to the stochastic differential equation
\begin{align*}
dX_t = b(t, X_t) \, dt + \sigma(t, X_t) \, dB_t, \quad X_0 = x_0,
\end{align*}
where \(b: [0, T] \times \mathbb{R}^m \to \mathbb{R}^m\) and \(\sigma: [0, T] \times \mathbb{R}^m \to \mathbb{R}^{m \times d}\) are sufficiently smooth functions, and \(B_t\) is a \(d\)-dimensional Brownian motion. Define the first variation process
\begin{align*}
dY_t = \partial_x b(t, X_t) Y_t \, dt + \sum_{l=1}^d \partial_x \sigma^l(t, X_t) Y_t \, dB_t^l, \quad Y_0 = I_m,
\end{align*}
and the Malliavin covariance matrix of $X_T$ as 
\begin{equation*}
 \gamma_{X_T} = \int_0^T Y_T Y_s^{-1} \sigma(s, X_s) \sigma(s, X_s)^\top (Y_s^{-1})^\top Y_T^\top \, ds
\end{equation*}
which we assume to be invertible almost surely. Furthermore, define 
the random field \( u_t(x) \) and the random vector \( F_k \) as follows\footnote{Note that \( u_t(x) \in \mathbb{R}^d \) since \( x^\top \) is a \( 1 \times m \) row vector, \( Y_t^{-1} \) is an \( m \times m \) matrix, and \( \sigma(t, X_t) \) is an \( m \times d \) matrix, resulting in a \( 1 \times d \) row vector. Moreover, \( u_t(x) \) is adapted to the natural Brownian filtration \( \mathcal{F}_t \) for each fixed \( x \), as it depends only on \( Y_t^{-1} \) and \( \sigma(t, X_t) \), both of which are adapted processes.}
\begin{align*}
u_t(x) &= x^\top Y_t^{-1} \sigma(t, X_t), \quad x \in \mathbb{R}^m, \\
F_k &= Y_T^\top \gamma_{X_T}^{-1} e_k, \quad \text{with} \quad F_k^j = e_j^\top Y_T^\top \gamma_{X_T}^{-1} e_k,
\end{align*}
where \( Y_t \) is the first variation process, \( Y_T \) is its value at time \( T \), \( Y_t^{-1} \) its inverse at time \( t \), and \( e_k \) is the \( k \)-th standard  basis vector in \( \mathbb{R}^m \).
%
%
Then, following \cite{Nualart_Nualart_2018}, for each \( k = 1, \dots, m \), the gradient of the log-density satisfies
\begin{align}
\partial_{y_k} \log {p_T}(y) = -\mathbb{E} \left[ \delta(u_k) \mid X_T = y \right],
\label{score}
\end{align}
where \( \delta(u_k) \) denotes the Skorokhod integral of the process \( u_k(t) \), i.e., the Malliavin divergence, whose explicit representation is given by
\begin{align}
\delta(u_k) &= 
\left. \int_0^T u_t(x) \cdot dB_t \right|_{x=F_k}- \int_0^T \sum_{j=1}^m \bigl[ Y_t^{-1} \sigma(t, X_t) \bigr]_j \cdot A_{jk}(t) \, dt \nonumber\\
&\quad + \int_0^T \sum_{j=1}^m \bigl[ Y_t^{-1} \sigma(t, X_t) \bigr]_j \cdot B_{jk}(t) \, dt + \int_0^T \sum_{j=1}^m \bigl[ Y_t^{-1} \sigma(t, X_t) \bigr]_j \cdot C_{jk}(t) \, dt,
\label{score1}
\end{align}

\text{where}

\begin{align*}
A_{jk}(t) &= e_j^\top \bigg[ \sigma(t, X_t)^\top (Y_t^{-1})^\top Z_T^\top - \sigma(t, X_t)^\top (Y_t^{-1})^\top Z_t^\top (Y_t^{-1})^\top Y_T^\top \\
&\quad\quad\quad + \left( Y_T Y_t^{-1} \partial_x \sigma(t, X_t) Y_t \right)^\top \bigg] \gamma_{X_T}^{-1} e_k, \\[0.5em]
B_{jk}(t) &= e_j^\top Y_T^\top \gamma_{X_T}^{-1} \cdot \bigg[ \int_0^t I_1(t,s) \, ds + \int_0^t I_2(t,s) \, ds \bigg] \gamma_{X_T}^{-1} e_k,\\[0.5em]
C_{jk}(t) &= e_j^\top Y_T^\top \gamma_{X_T}^{-1} \cdot \bigg[ \int_t^T I_3(t,s) \, ds + \int_t^T I_4(t,s) \, ds \bigg] \gamma_{X_T}^{-1} e_k,
\end{align*}

\text{with integrands} ($p$ and $q$ below denote components)

\begin{align*}
I^{p,q}_1(t,s) &= \bigg[ \Omega(t) Y_s^{-1}\sigma(s, X_s) \bigg]^p \cdot [Y_T Y_s^{-1} \sigma(s, X_s)]^q, \\[0.5em]
I^{p,q}_2(t,s) &= [Y_T Y_s^{-1} \sigma(s, X_s)]^p \cdot \bigg[ \Omega(t) Y_s^{-1}\sigma(s, X_s) \bigg]^q, \\[0.5em]
I^{p,q}_3(t,s) &= \bigg[ \Omega(t) Y_s^{-1}\sigma(s, X_s) + Y_T \Theta(t,s) \bigg]^p \cdot [Y_T Y_s^{-1} \sigma(s, X_s)]^q, \\[0.5em]
I^{p,q}_4(t,s) &= [Y_T Y_s^{-1} \sigma(s, X_s)]^p \cdot \bigg[ \Omega(t) Y_s^{-1}\sigma(s, X_s) + Y_T \Theta(t,s) \bigg]^q,
\end{align*}

\text{and}

\begin{align*}
\Omega(t) &= Z_T Y_t^{-1} \sigma(t, X_t) - Y_T Y_t^{-1} Z_t Y_t^{-1} \sigma(t, X_t) + Y_T Y_t^{-1} \partial_x \sigma(t, X_t) Y_t, \\[0.5em]
\Theta(t,s) &= -Y_s^{-1} \bigg[ Z_s Y_t^{-1} \sigma(t, X_t) - Y_s Y_t^{-1} Z_t Y_t^{-1} \sigma(t, X_t) \\
&\quad + Y_s Y_t^{-1} \partial_x \sigma(t, X_t) Y_t \bigg] Y_s^{-1} \sigma(s, X_s) + Y_s^{-1} \partial_x \sigma(s, X_s) \bigg(Y_s Y_t^{-1} \sigma(t, X_t) \bigg).
\end{align*}
The second variation process appearing in $\Omega(t)$ and $\Theta(t,s)$ is defined as
\begin{align*}
dZ_t &= \left[ \partial_{xx} b(t, X_t) (Y_t \otimes Y_t) + \partial_x b(t, X_t) Z_t \right] dt \\
&\quad + \sum_{l=1}^d \left[ \partial_{xx} \sigma^l(t, X_t) (Y_t \otimes Y_t) + \partial_x \sigma^l(t, X_t) Z_t \right] dB_t^l,
\end{align*}
with initial condition \( Z_0 = 0 \).
\end{theorem}

\noindent
The formula \eqref{score1} for the Skorokhod integral $\delta(u_k)$ consists of several components that capture the interaction between stochastic processes and their Malliavin derivatives. The first term $\left. \int_0^T u_t(x) \cdot dB_t \right|_{x=F_k}$ represents the It\^o integral from Bismut's formula, evaluated at the specific point $x = F_k$. This term captures the direct influence of the Brownian motion on the process. The remaining components in the expansion of $\sum_{j=1}^m \partial_j u_t(F_k)\,\cdot\,D_t F_k^j$ are subtracted from the first component and are indexed by the triplet $(j,p,q) \in \{1,\ldots,m\}^3$. This indexing arises from extracting matrix components of the inverse Malliavin matrix $\gamma_{X_T}^{-1}$ and selecting vector components of $Y_t^{-1}\,\sigma(t,X_t)$. The formula incorporates integrals over different time intervals, i.e., integrals over $s \in [0,t]$ and integrals over $s \in [t,T]$. These appear in the $D_t\gamma_{X_T}^{p,q}$ expansions and represent how the Malliavin derivative propagates through the system depending on whether the perturbation time $t$ occurs before or after the reference time $s$.

Note that Theorem \ref{thm:main} provides an explicit representation of \(\nabla_y \log p_t(y)\) that avoids abstract Malliavin derivatives, and therefore it is computationally practical for applications in score-based generative modelling.

\section{Malliavin calculus and covering vector fields}
\label{sec:malliavin}
Malliavin calculus enables us to analyse the regularity of the random process \(X_T\) and compute derivatives of its density. It also allows us to represent score functions via the Bismut-type formula\footnote{The negative sign arises from the integration-by-parts formula in Malliavin calculus, and the expectation is conditioned on \(X_T = y\), reflecting the evaluation of the density gradient at a specific point.}
\begin{equation}
\partial_{y_k} \log p_T(y) = -\E \left[ \delta(u_k) \mid X_T = y \right],\qquad k=1,\ldots, m
\end{equation}
where \(\delta(u_k)\) is the Skorokhod integral and \(u_k = \{ u_k(t) : 0 \leq t \leq T \}\), with \(u_k(t) \in \R^d\), 
is a process called the {\em covering vector field}, \cite{malliavin.thalmaier:06:stochastic, Nualart_Nualart_2018}. 
The covering vector field  \(u_k(t)\) must satisfy
\begin{equation}
\langle D X_T^i, u_k \rangle_H = \delta_{ik}, \quad i, k = 1, \ldots, m,
\label{cond}
\end{equation}
where  \(\delta_{ik} \) is the Kronecker delta, \(D X_T^i = \{ D_t X_T^i : 0 \leq t \leq T \}\) is the Malliavin derivative of \(X_T^i\), a process in the Hilbert space \(H = L^2\left([0, T], \R^d\right)\), 
\begin{equation*}
\langle f, g \rangle_H = \int_0^T f(t) \cdot g(t) \, dt
\end{equation*}
 is the inner product in \(H\), and \(\cdot\) denotes the standard dot product in \(\R^d\).
Condition \eqref{cond} ensures that \(u_k(t)\) ``covers'' the \(k\)-th direction 
in the Malliavin sense, which will allow us to isolate \(\partial_{y_k} p(y)\).
Define the Malliavin covariance matrix \(\gamma_{X_T} \in \R^{m \times m}\) with entries
\begin{equation}
\gamma_{X_T}^{i,j} = \langle D X_T^i, D X_T^j \rangle_H,
\label{MalliavinCov}
\end{equation}
where \(D_t X_T^i \in \R^d\) is the Malliavin derivative of \(X_T^i\) at time \(t\). This matrix measures the ``randomness'' induced by the Brownian motion up to time \(T\). We assume \(\gamma_{X_T}\) is invertible almost surely, so \(\gamma_{X_T}^{-1}\) exists, which 
is necessary for \(p(y)\) to be smooth.
%
%
Hereafter we show that for each \(k = 1, \ldots, m\), the process
\begin{equation}
u_k(t) = \sum_{j=1}^m \left( \gamma_{X_T}^{-1} \right)_{k,j} D_t X_T^j,
\label{covering}
\end{equation}
where \(\left( \gamma_{X_T}^{-1} \right)_{k,j}\) is the \((k, j)\)-th element of the inverse Malliavin matrix \(\gamma_{X_T}^{-1}\), satisfies the covering condition
\begin{equation*}
\langle D X_T^i, u_k \rangle_H = \delta_{ik} \quad \text{for all} \quad i = 1, \ldots, m.
\end{equation*}
\noindent
In the remainder of this section, we state and prove several important properties of the chosen covering vector field \eqref{covering}. Although some of these results may already be known or established in prior works, we include proofs here for the reader’s convenience and to ensure the exposition remains self-contained, for more details, see \cite{kunita1997stochastic, nualart2006malliavin}.

We begin with a theorem establishing the existence and uniqueness of a covering vector field \( u_k(t) \) for an \( \mathbb{R}^m \)-valued random variable \( X_T = (X_T^1, \ldots, X_T^m) \), defined on a probability space \( (\Omega, \mathcal{F}, \mathbb{P}) \) with a \( d \)-dimensional Brownian motion \( B=\{B_t\}_{t \in [0, T]} \). The theorem assumes that \( X_T \) lies in \( \mathbb{D}^{2,2} \), the space of twice Malliavin differentiable random variables relative to the Gaussian structure of \( B \), and that its Malliavin covariance matrix \( \gamma_{X_T} \), with entries \( \gamma_{X_T}^{i,j} = \langle D X_T^i, D X_T^j \rangle_H \) in the Cameron–Martin space \( H = L^2([0, T]; \mathbb{R}^d) \), is invertible almost surely and satisfies \( \mathbb{E}[\lambda_{\min}(\gamma_{X_T})^{-p}] < \infty \) for some \( p > 2(m+1) \). Under these conditions, for each \( k = 1, \ldots, m \), it can be shown that there exists a unique energy-minimising process \( u_k(t) \in L^2([0, T] \times \Omega; \mathbb{R}^d) \) in the domain of the Skorokhod integral \( \mathrm{Dom}(\delta) \), satisfying \( \langle D X_T^i, u_k \rangle_H = \delta_{ik} \), where \( \delta_{ik} \) is the Kronecker delta.

\begin{theorem}[Existence and uniqueness of the covering vector field]
Let $X_T = (X_T^1, \ldots, X_T^m)$ be an $\mathbb{R}^m$-valued random variable defined on a probability space $(\Omega, \mathcal{F}, \mathbb{P})$, equipped with a Brownian motion $B = \{B_t\}_{t \in [0, T]}$ taking values in $\mathbb{R}^d$. Suppose $X_T \in \mathbb{D}^{2,2}$. Let $\gamma_{X_T}$ be the Malliavin covariance matrix of $X_T$ defined in \eqref{MalliavinCov}. Assume $\gamma_{X_T}$ is invertible almost surely with $\mathbb{E}[\lambda_{\min}(\gamma_{X_T})^{-p}] < \infty$ for some $p > 2(m+1)$.

Then, for each $k = 1, \ldots, m$, there exists a unique process $u_k = \{u_k(t)\}_{t \in [0, T]}$ in $\text{Dom}(\delta) \cap L^2([0, T] \times \Omega, \mathbb{R}^d)$ minimising $\mathbb{E}[\int_0^T \|u(t)\|^2 \, dt]$ such that
\begin{equation}
\langle D X_T^i, u_k \rangle_H = \delta_{ik}, \quad i = 1, \ldots, m.
\label{cond0}
\end{equation}
\end{theorem}

\begin{proof}
Set $H = L^2([0, T]; \mathbb{R}^d)$, $\gamma = \gamma_{X_T}$, and define the covering field by
\[
u_k(t) = \sum_{j=1}^m \left(\gamma^{-1}\right)_{jk} D_t X_T^j.
\]

Define the operator $T: H \to \mathbb{R}^m$ by $Th = (\langle DX_T^i, h\rangle_H)_{i=1}^m$. Its adjoint satisfies $T^*\mathbf{c} = \sum_i c_i DX_T^i$, hence $TT^* = \gamma$. Since $TT^* = \gamma$, we have
\[
u_k = T^*\gamma^{-1}\mathbf{e}_k,
\]
which compresses verification and minimality into a standard Moore-Penrose argument.

Let $\mathbf{v}_k = (\left(\gamma^{-1}\right)_{1k}, \ldots, \left(\gamma^{-1}\right)_{mk})^\top$. Then
\begin{align}
\|u_k(t)\|^2 &= \sum_{j,\ell} \left(\gamma^{-1}\right)_{jk} \left(\gamma^{-1}\right)_{\ell k} (D_t X_T^j \cdot D_t X_T^\ell) \quad \text{(Euclidean inner product)},\\
\int_0^T \|u_k(t)\|^2 \, dt &= \mathbf{v}_k^\top \gamma \mathbf{v}_k.
\end{align}
Since $\gamma \mathbf{v}_k = \mathbf{e}_k$, we have $\mathbf{v}_k^\top \gamma \mathbf{v}_k = \left(\gamma^{-1}\right)_{kk}$. Thus
\[
\mathbb{E}\left[\int_0^T \|u_k(t)\|^2 \, dt\right] = \mathbb{E}[\left(\gamma^{-1}\right)_{kk}] \leq \mathbb{E}[\lambda_{\min}(\gamma)^{-1}].
\]
By H\"older’s inequality with conjugate exponents \(p\) and \({p}/{(p-1)}\),
\[
\mathbb{E}\!\left[\lambda_{\min}(\gamma)^{-1}\right]
\le \Big(\mathbb{E}\!\left[\lambda_{\min}(\gamma)^{-p}\right]\Big)^{1/p}
< \infty .
\]
Suppose $v_k$ also satisfies \eqref{cond0}. Then $w = u_k - v_k \in \ker T$. Since $u_k \in \text{range}(T^*)$ by construction and $\ker T = \text{range}(T^*)^\perp$, we have $\langle u_k, w \rangle_H = 0$. Therefore,
\[
\|v_k\|^2_{L^2} = \|u_k - w\|^2_{L^2} = \|u_k\|^2_{L^2} + \|w\|^2_{L^2} \geq \|u_k\|^2_{L^2},
\]
with equality only if $w = 0$. The minimality condition forces $w = 0$, establishing uniqueness.
Since $X_T \in \mathbb{D}^{2,2}$, each $\gamma^{ij} \in \mathbb{D}^{1,2}$. The cofactors of $\gamma$ inherit this regularity, and Cramer's formula gives
\[
\left(\gamma^{-1}\right)_{jk} = \frac{\text{adj}(\gamma)_{jk}}{\det \gamma}.
\]
The moment bound $\mathbb{E}[\lambda_{\min}(\gamma)^{-p}] < \infty$ with $p > 2(m+1)$ ensures the required integrability via the Malliavin chain rule after localisation (truncate $\det \gamma$ away from 0 and pass to the limit). Since the inverse matrix entries satisfy $\left(\gamma^{-1}\right)_{jk} \in \mathbb{D}^{1,2}$ and $D_t X_T^j \in \mathbb{D}^{1,2}(H)$, the product rule yields $u_k \in \mathbb{D}^{1,2}(H) \subset \text{Dom}(\delta)$.

\end{proof}

\noindent 
The next theorem concerns the regularity of the covering vector field \(u_k(t)\) for a stochastic process \(\bigl(X_t\bigr)_{t\in[0,T]}\) adapted to a Brownian filtration.  
It assumes that the terminal value \(X_T\) lies in \(\mathbb{D}^{2,2}\), that the Malliavin covariance matrix \(\gamma_{X_T}\) is almost surely invertible, and it defines \(u_k(t)\) as in~\eqref{covering}.  
Additional regularity conditions are: uniform ellipticity of \(\gamma_{X_T}\) (i.e.\ \(\gamma_{X_T}\ge\lambda I\) for some \(\lambda>0\)); and, when \(X_t\) solves the SDE \(\mathrm dX_t = b(t,X_t)\,\mathrm dt + \sigma(t,X_t)\,\mathrm dB_t\), the coefficients \(b\) and \(\sigma\) are \(\mathcal{C}^2\) with all derivatives up to second order bounded.  
Furthermore, under the uniform ellipticity bound one automatically has \(\mathbb{E}\bigl[\|\gamma_{X_T}^{-1}\|^{\,q}\bigr]<\infty\) for every \(q>0\).  
Under these hypotheses the theorem shows that each \(u_k(t)\) belongs to \(\Dom(\delta)\) and that its Skorokhod integral satisfies \(\delta(u_k)\in L^{p}(\Omega)\) for some \(p>1\), thereby confirming the desired regularity of the vector field.

\begin{theorem}[Regularity of the covering vector field]
Let \(X = (X_t)_{t\in[0,T]}\) be the unique solution of
\[
  \mathrm dX_t = b(t,X_t)\,\mathrm dt + \sigma(t,X_t)\,\mathrm dB_t,
\]
where \(b,\sigma\in \mathcal{C}^2\) with all mixed derivatives up to order two bounded.  Suppose further that
\begin{enumerate}[label=(\roman*)]
  \item \(X_T\in\mathbb D^{2,2}\), and its Malliavin covariance matrix \eqref{MalliavinCov}
    is a.s.\ invertible with
    \(\gamma_{X_T}\ge\lambda I_m\) for some \(\lambda>0\).
  \item Define, for each \(k=1,\dots,m\), the covering vector field defined at \eqref{covering}
\end{enumerate}
Then each \(u_k\) belongs to \(\Dom(\delta)\), and its Skorokhod integral
\(\delta(u_k)\) is well-defined with \(\delta(u_k)\in L^p(\Omega)\) for some \(p>1\).
\end{theorem}

\begin{proof}
The Skorokhod integral \(\delta\) is the adjoint of the Malliavin derivative operator. A sufficient condition for \( u \in \text{Dom}(\delta) \) is that \( u \in \mathbb{D}^{1,2}(H) \) with
\[
\| u \|_{\mathbb{D}^{1,2}(H)}^2 = \mathbb{E} \left[ \int_0^T |u(t)|^2 \, dt \right] + \mathbb{E} \left[ \int_0^T \int_0^T |D_s u(t)|^2 \, ds \, dt \right] < \infty.
\]
We verify that \( u_k \in \mathbb{D}^{1,2}(H) \). The Malliavin derivative of \( u_k(t) \) is
\[
D_s u_k(t) = \sum_{j=1}^m \left[ D_s \left( (\gamma_{X_T}^{-1})_{j k} \right) D_t X_T^j + (\gamma_{X_T}^{-1})_{j k} D_s D_t X_T^j \right],
\]
where
\[
D_s (\gamma_{X_T}^{-1})_{j k} = - \sum_{p, q=1}^m (\gamma_{X_T}^{-1})_{j p} (D_s (\gamma_{X_T})_{p q}) (\gamma_{X_T}^{-1})_{q k},
\]
and
\[
D_s (\gamma_{X_T})_{p q} = \int_0^T \left[ D_s D_t X_T^p \cdot D_t X_T^q + D_t X_T^p \cdot D_s D_t X_T^q \right] dt.
\]
Since \( X_T \in \mathbb{D}^{2,2} \) and \( \|\gamma_{X_T}^{-1}\| \leq \lambda^{-1} \), we have \( D_s (\gamma_{X_T}^{-1})_{j k} \in L^2(\Omega \times [0, T]) \).
We need to verify the following two conditions
\begin{enumerate}
    \item[a)] \( \mathbb{E} \left[ \int_0^T |u_k(t)|^2 \, dt \right] < \infty \),
    \item[b)] \( \mathbb{E} \left[ \int_0^T \int_0^T |D_s u_k(t)|^2 \, ds \, dt \right] < \infty \).
\end{enumerate}
For the first condition
\[
\int_0^T |u_k(t)|^2 \, dt \leq \| \gamma_{X_T}^{-1} \|^2 \sum_{j=1}^m \int_0^T |D_t X_T^j|^2 \, dt = \left\| \gamma_{X_T}^{-1} \right\|^2 \operatorname{Tr}\left( \gamma_{X_T}\right).
\]
Since \( \| \gamma_{X_T}^{-1} \| \leq \lambda^{-1} \) and \( \mathbb{E} \left[ \operatorname{Tr} \gamma_{X_T} \right] < \infty \) (as \( X_T \in \mathbb{D}^{1,2} \)), we obtain
\[
\mathbb{E} \left[ \int_0^T |u_k(t)|^2 \, dt \right] \leq \lambda^{-2} \mathbb{E} \left[ \operatorname{Tr} \gamma_{X_T} \right] < \infty.
\]
For the second condition
\begin{align}
|D_s u_k(t)|^2 \leq 2 \left| \sum_{j=1}^m D_s (\gamma_{X_T}^{-1})_{j k} D_t X_T^j \right|^2 + 2 \left| \sum_{j=1}^m (\gamma_{X_T}^{-1})_{j k} D_s D_t X_T^j \right|^2.
\end{align}
Using Cauchy-Schwarz and integrating
\begin{align*}
\mathbb{E} \left[ \int_0^T \int_0^T |D_s u_k(t)|^2 \, ds \, dt \right]&\leq 2m \sum_{j=1}^m \mathbb{E} \bigg[ \int_0^T \int_0^T \Big( |D_s (\gamma_{X_T}^{-1})_{j k}|^2 |D_t X_T^j|^2 \\
&\quad + |(\gamma_{X_T}^{-1})_{j k}|^2 |D_s D_t X_T^j|^2 \Big) ds \, dt \bigg].
\end{align*}
Since \( |D_s (\gamma_{X_T}^{-1})_{j k}|^2 \leq C \|D_s \gamma_{X_T}\|^2 \in L^2(\Omega \times [0,T]) \), \( |D_t X_T^j|^2 \in L^2(\Omega \times [0,T]) \), and \( |(\gamma_{X_T}^{-1})_{j k}|^2 \leq \lambda^{-2} \), while \( X_T \in \mathbb{D}^{2,2} \) ensures \( \mathbb{E} \left[ \int_0^T \int_0^T |D_s D_t X_T^j|^2 \, ds \, dt \right] < \infty \), the second condition holds.
Therefore, \( u_k \in \mathbb{D}^{1,2}(H) \subseteq \text{Dom}(\delta) \) with \( \delta(u_k) \in L^2(\Omega) \).
For \( p = 2 \), we verify \( u_k \in \mathbb{D}^{1,2}(H) \) directly
\[
\mathbb{E}\left[ \int_0^T |u_k(t)|^2 dt \right] \leq \lambda^{-2} \mathbb{E}\left[ \operatorname{Tr}(\gamma_{X_T}) \right] < \infty,
\]
and similarly for the Malliavin derivative term. Since the Skorokhod integral \(\delta: \mathbb{D}^{1,2}(H) \to L^2(\Omega)\) is continuous, we have \(\delta(u_k) \in L^2(\Omega) \subseteq L^p(\Omega)\) for \(p = 2 > 1\).

\end{proof}

\noindent
In the next theorem we establish the \( L^2 \)-continuity of the covering vector field associated with a stochastic process \( (X_t)_{t \in [0, T]} \), adapted to a Brownian filtration. The process is assumed to have a terminal value \( X_T \) in the Malliavin space \( \mathbb{D}^{2,2} \), with an invertible Malliavin covariance matrix \( \gamma_{X_T} \). The covering vector field \( u_k(t) \) is defined using \( \gamma_{X_T}^{-1} \) and the Malliavin derivatives of \( X_T \). Additionally, \( X_t \) satisfies an SDE with \( \mathcal{C}^2 \) coefficients \( b \) and \( \sigma \) having bounded derivatives up to second order. Under uniform ellipticity of \( \gamma_{X_T} \), the theorem proves that \( t \mapsto u_k(t) \) is continuous into \( L^2(\Omega) \), highlighting the regularity of the vector field in the stochastic setting.


\begin{theorem}[{$L^2$}–continuity of the covering vector field]
Let $(X_t)_{t\in[0,T]}$ be the unique solution of
\[
  \mathrm dX_t \;=\; b(t,X_t)\,\mathrm dt \;+\;\sigma(t,X_t)\,\mathrm dB_t,
  \quad X_0=x_0,
\]
where $b,\sigma\in \mathcal{C}^2\bigl([0,T]\times\mathbb{R}^m\bigr)$ have all first and second derivatives bounded, and
$\sigma$ satisfies the linear growth bound $\|\sigma(t,x)\|\le C(1+\|x\|)$.  
Suppose additionally that
\begin{enumerate}[label=(\roman*)]
  \item $X_T\in\mathbb{D}^{2,2}$ and its Malliavin covariance \eqref{MalliavinCov}
    is a.s.\ invertible with
    $\gamma_{X_T}\ge\lambda I_m$ for some $\lambda>0$;
  \item for each $t \in [0,T]$, the mapping $x \mapsto X_t^x$ is a random diffeomorphism of class $\mathcal{C}^2$ almost surely, and for some $p \geq 8$,
    \[
    \mathbb{E}\left[\sup_{u\in[0,T]} \|Y_u\|^p + \sup_{u\in[0,T]} \|Y_u^{-1}\|^p\right] < \infty.
    \]
\end{enumerate}
Define for each $k=1,\dots,m$ the covering field \eqref{covering}.
Then $t\mapsto u_k(t)$ is continuous as a map
$[0,T]\to L^2(\Omega)$; equivalently
\[
  \lim_{s\to t}\mathbb{E}\bigl[|u_k(s)-u_k(t)|^2\bigr]\;=\;0
  \quad\forall\,t\in[0,T].
\]
\end{theorem}
\begin{proof}
Since $b,\sigma\in\mathcal{C}^2$ with bounded derivatives, $X_t\in\mathbb{D}^{2,2}$. The Malliavin derivative $D_tX_T$ is given by
\[
D_t X_T = Y_T Y_t^{-1} \sigma(t, X_t) \mathbf{1}_{\{t \leq T\}},
\]
where $Y_t = \partial X_t/\partial x$ is the first variation process satisfying
\[
dY_t = \partial_x b(t, X_t) Y_t \, dt + \sum_{l=1}^d \partial_x \sigma_l(t, X_t) Y_t \, dB_t^l, \quad Y_0 = I_m.
\]
The process $Y_t$ is continuous and invertible a.s.\ with $\mathbb{E}[\|Y_T\|^2] < \infty$.
Note explicitly that $t \mapsto Y_t^{-1}$ is a.s.\ continuous since the inverse map is continuous on $\mathrm{GL}_m$ and $t \mapsto Y_t$ is a.s.\ continuous.
Also, under our global Lipschitz/linear-growth hypotheses,
\[
\mathbb{E}\left[\sup_{u\in[0,T]} \|X_u\|^p\right] < \infty \quad \text{for some } p \geq 8.
\]
For $s,t \leq T$, we have
\[
D_s X_T - D_t X_T = Y_T \left( Y_s^{-1} \sigma(s, X_s) - Y_t^{-1} \sigma(t, X_t) \right).
\]
Since $Y_t^{-1}\sigma(t,X_t)$ is continuous in $t$ (by continuity of $X_t$, $Y_t$, and $\mathcal{C}^2$-regularity of $\sigma$), and
\begin{align*}
\|Y_s^{-1} \sigma(s, X_s)\| &\leq C \left(1 + \sup_{u\leq T} \|X_u\|\right) \sup_{u\leq T} \|Y_u^{-1}\|,\\
|D_s X_T - D_t X_T| &\leq \|Y_T\| \|Y_s^{-1} \sigma(s, X_s) - Y_t^{-1} \sigma(t, X_t)\|\\
&\leq 2C \|Y_T\| \left(1 + \sup_{u\leq T} \|X_u\|\right) \sup_{u\leq T} \|Y_u^{-1}\|,
\end{align*}
hence
\begin{align*}
|D_s X_T - D_t X_T|^2 \leq C' \|Y_T\|^2 \left(1 + \sup_{u\leq T} \|X_u\|\right)^2 \left(\sup_{u\leq T} \|Y_u^{-1}\|\right)^2.
\end{align*}
By (ii) with $p \geq 8$ and the standard moment bound $\mathbb{E}[\sup_{u\leq T} \|X_u\|^p] < \infty$, Hölder with exponents $(2,4,4)$ gives the RHS integrable, so dominated convergence theorem applies to yield $\mathbb{E}[|D_s X_T - D_t X_T|^2] \to 0$. 
By assumption, $\gamma_{X_T} \geq \lambda I_m$ a.s., so $\|\gamma_{X_T}^{-1}\| \leq \lambda^{-1}$. For the covering field,
\[
u_k(s) - u_k(t) = \sum_{j=1}^m (\gamma_{X_T}^{-1})_{jk} (D_s X_T^j - D_t X_T^j),
\]
hence by Cauchy-Schwarz,
\begin{align*}
\mathbb{E} \left[ |u_k(s) - u_k(t)|^2 \right] &\leq \mathbb{E} \left[ \|\gamma_{X_T}^{-1} e_k\|^2 \sum_{j=1}^m |D_s X_T^j - D_t X_T^j|^2 \right] \\
&\leq \lambda^{-2} \sum_{j=1}^m \mathbb{E} \left[ |D_s X_T^j - D_t X_T^j|^2 \right].
\end{align*}
Since each term converges to zero as $s \to t$, we obtain $\lim_{s\to t}\mathbb{E}[|u_k(s)-u_k(t)|^2] = 0$.

\end{proof}

Next we present a stability theorem for covering vector fields associated with stochastic processes under controlled perturbations. The theorem examines a sequence of stochastic processes \( (X_t^n)_{t \in [0, T]} \) and a limiting process \( (X_t)_{t \in [0, T]} \), described by stochastic differential equations driven by the same Brownian motion, with coefficients \( b_n, \sigma_n \) and \( b, \sigma \), respectively, where these coefficients satisfy smoothness and boundedness conditions (they are \( \mathcal{C}^2 \) with first and second derivatives bounded uniformly in \(n\)) and converge locally uniformly together with their spatial derivatives (i.e., \(C^1_{\mathrm{loc}}\)-convergence). Both processes are adapted to this Brownian filtration, their terminal values \(X_T^n\) and \(X_T\) lie in the Malliavin Sobolev space \( \mathbb{D}^{1,2} \), and their Malliavin covariance matrices are almost surely invertible with a uniform ellipticity bound; moreover, \(\sup_n \mathbb{E}\!\left[\|(\gamma^n)^{-1}\|^2\right]<\infty\) and \(\mathbb{E}\!\left[\|\gamma^{-1}\|^2\right]<\infty\). We also assume uniform moment bounds for the first-variation processes \(Y_t^n=\partial_x X_t^n\) and \((Y_t^n)^{-1}\), and uniform sup-in-time \(L^2\) bounds for \(X^n\) and \(X\). In addition, we impose a uniform \(L^4\) Malliavin control, i.e., \(\sup_n \mathbb{E}\!\int_0^T \|D_t X_T^n\|^4\,dt < \infty\) and \(\mathbb{E}\!\int_0^T \|D_t X_T\|^4\,dt < \infty\) (equivalently, \(X_T^n,X_T\in\mathbb{D}^{1,4}\)). The theorem proves that the covering vector fields \( u_k^n(t) \), defined using the inverse Malliavin covariance matrices and Malliavin derivatives, converge to \( u_k(t) \) in the \( L^2([0, T] \times \Omega) \) norm as \( n \to \infty \). This result highlights the robustness of the covering vector field construction against perturbations in the SDE dynamics.

\begin{theorem}[Stability under perturbations]
Let \(X^n\) and \(X\) be the unique strong solutions on \([0,T]\) of
\[
  \mathrm dX^n_t = b_n(t,X^n_t)\,\mathrm dt + \sigma_n(t,X^n_t)\,\mathrm dB_t,
  \qquad
  \mathrm dX_t   = b(t,X_t)\,\mathrm dt + \sigma(t,X_t)\,\mathrm dB_t,
\]
driven by the same \(d\)--dimensional Brownian motion \(B\).  Assume:
\begin{enumerate}[label=(\roman*)]
  \item We have $b_n, b, \sigma_n, \sigma \in C^2([0,T] \times \mathbb{R}^m)$ with all first and second derivatives 
    bounded uniformly in $n$, and $(b_n, \sigma_n) \to (b, \sigma)$ locally uniformly, and 
    $\partial_x b_n \to \partial_x b$, $\partial_x \sigma_n \to \partial_x \sigma$ locally uniformly 
    (i.e., $C^1_{\mathrm{loc}}$-convergence).
    
  \item For each $n$ and for $X$, the terminal value lies in $\mathbb{D}^{1,2}$, and their Malliavin 
    covariance matrices are given by
    $\gamma^n = \int_0^T D_t X_T^n (D_t X_T^n)^\top \, dt$ and 
    $\gamma = \int_0^T D_t X_T (D_t X_T)^\top \, dt$,
    which satisfy $\gamma^n, \gamma \geq \lambda I_m > 0$ a.s.\ for some $\lambda > 0$, together with
    $\sup_n \mathbb{E}[\|(\gamma^n)^{-1}\|^2] < \infty$ and $\mathbb{E}[\|\gamma^{-1}\|^2] < \infty$.
    
  \item The first variation processes $Y_t^n = \partial_x X_t^n$ satisfy the bounds
    $\sup_{n,t} \mathbb{E}[\|Y_t^n\|^2 + \|(Y_t^n)^{-1}\|^2] < \infty$, and we also have
    $\sup_n \mathbb{E}\big[\sup_{t \in [0,T]} \|X_t^n\|^2\big] < \infty$ and 
    $\mathbb{E}\big[\sup_{t \in [0,T]} \|X_t\|^2\big] < \infty$.
    
  \item \textit{(Uniform $L^4$ Malliavin control)} We require that
    $\sup_n \mathbb{E} \int_0^T \|D_t X_T^n\|^4 \, dt < \infty$ and 
    $\mathbb{E} \int_0^T \|D_t X_T\|^4 \, dt < \infty$.
\end{enumerate}
Define, for each \(k=1,\dots,m\),
\[
  u^n_k(t)
  \;=\;\sum_{j=1}^m(\gamma^n)^{-1}_{jk}\,D_tX_T^{n,j},
  \quad
  u_k(t)
  \;=\;\sum_{j=1}^m\gamma^{-1}_{jk}\,D_tX_T^j.
\]
Then for every \(k\),
\[
  \lim_{n\to\infty}
  \mathbb{E}\!\Bigl[\!\int_0^T\bigl|u^n_k(t)-u_k(t)\bigr|^2\,dt\Bigr]
  \;=\;0.
\]
\end{theorem}

\begin{proof}
The proof proceeds in four steps: establishing convergence of \(X_T^n \to X_T\), \(D_t X_T^n \to D_t X_T\), \((\gamma^n)^{-1} \to \gamma^{-1}\), and finally \(u_k^n \to u_k\).
%
To show \(\mathbb{E}[|X_T^n - X_T|^2] \to 0\), we write
\[
X_t^n - X_t = \int_0^t [b_n(s, X_s^n) - b(s, X_s)] ds + \int_0^t [\sigma_n(s, X_s^n) - \sigma(s, X_s)] dB_s.
\]
Using \((a + b)^2 \leq 2a^2 + 2b^2\) and taking suprema:
\begin{align*}
\mathbb{E}\left[\sup_{t \in [0,T]} |X_t^n - X_t|^2\right] &\leq 2\mathbb{E}\left[\sup_{t \in [0,T]} \left|\int_0^t [b_n(s, X_s^n) - b(s, X_s)] ds\right|^2\right] \\
&\quad + 2\mathbb{E}\left[\sup_{t \in [0,T]} \left|\int_0^t [\sigma_n(s, X_s^n) - \sigma(s, X_s)] dB_s\right|^2\right].
\end{align*}
For the drift term, by Cauchy-Schwarz
\[
\left(\int_0^T |b_n(s, X_s^n) - b(s, X_s)| ds\right)^2 \leq T \int_0^T |b_n(s, X_s^n) - b(s, X_s)|^2 ds.
\]
For the diffusion term, by the Burkholder-Davis-Gundy inequality
\begin{align*}
\mathbb{E}\left[\sup_{t \in [0,T]} \left|\int_0^t [\sigma_n(s, X_s^n) - \sigma(s, X_s)] dB_s\right|^2\right] &\leq C_{\mathrm{BDG}} \mathbb{E}\left[\int_0^T \|\sigma_n(s, X_s^n) - \sigma(s, X_s)\|^2 ds\right].
\end{align*}
Since the convergence \((b_n,\sigma_n)\to(b,\sigma)\) is locally uniform, we employ a localisation argument. Define the stopping time \(\tau_R = \inf\{t : \max(\|X_t\|, \|X_t^n\|) > R\}\). On \([0,\tau_R]\), the processes remain in \(B_R\), where the convergence is uniform. Decomposing the differences using this uniform convergence on \(B_R\) and Lipschitz continuity
\begin{align*}
|b_n(s, X_s^n) - b(s, X_s)|^2 &\leq 2|b_n(s, X_s^n) - b(s, X_s^n)|^2 + 2|b(s, X_s^n) - b(s, X_s)|^2 \\
&\leq 2\sup_{(t,x) \in [0,T] \times B_R} |b_n(t, x) - b(t, x)|^2 + 2C^2|X_s^n - X_s|^2,
\end{align*}
and similarly for \(\sigma\). This yields
\begin{align*}
\mathbb{E}\left[\sup_{t \in [0,\tau_R]} |X_t^n - X_t|^2\right] &\leq 4T^2 \sup_{(t,x) \in [0,T] \times B_R} |b_n(t, x) - b(t, x)|^2 \\
&\quad + 4C_{\mathrm{BDG}}T \sup_{(t,x) \in [0,T] \times B_R} \|\sigma_n(t, x) - \sigma(t, x)\|^2 \\
&\quad + (4TC^2 + 4C_{\mathrm{BDG}}C^2) \int_0^T \mathbb{E}\left[\sup_{u \in [0,s \wedge \tau_R]} |X_u^n - X_u|^2\right] ds.
\end{align*}
Let 
\begin{align}
A_{n,R} &= 4T^2 \sup_{(t,x) \in [0,T] \times B_R} |b_n(t, x) - b(t, x)|^2 \notag\\
&\quad + 4C_{\mathrm{BDG}}T \sup_{(t,x) \in [0,T] \times B_R} \|\sigma_n(t, x) - \sigma(t, x)\|^2
\end{align}
and $K = 4TC^2 + 4C_{\mathrm{BDG}}C^2$. By Gronwall's inequality
\[
\mathbb{E}\left[\sup_{t \in [0,\tau_R]} |X_t^n - X_t|^2\right] \leq A_{n,R} e^{KT}.
\]
\[
\mathbb{E}|X_T^n-X_T|^2 
= \mathbb{E}\big[|X_T^n-X_T|^2\mathbf{1}_{\{\tau_R=T\}}\big]
  + \mathbb{E}\big[|X_T^n-X_T|^2\mathbf{1}_{\{\tau_R<T\}}\big],
\]
For fixed \(R\), \(A_{n,R} \to 0\) as \(n \to \infty\). By (iii), \(\sup_n \mathbb{E}\sup_{t\le T}\|X_t^n\|^2 + \mathbb{E}\sup_{t\le T}\|X_t\|^2 < \infty\), hence \(\mathbb{P}(\tau_R<T)\to0\) as \(R\to\infty\), which yields \(\mathbb{E}[|X_T^n-X_T|^2]\to0\).
%
The Malliavin derivative satisfies
\[
D_t X_T^n = Y_T^n (Y_t^n)^{-1} \sigma_n(t, X_t^n) \mathbf{1}_{\{t \leq T\}}, \quad D_t X_T = Y_T Y_t^{-1} \sigma(t, X_t) \mathbf{1}_{\{t \leq T\}},
\]
where the first variation processes satisfy
\[
dY_t^n = \partial_x b_n(t, X_t^n) Y_t^n \, dt + \sum_{l=1}^d \partial_x \sigma_{n,l}(t, X_t^n) Y_t^n \, dB_t^l, \quad Y_0^n = I_m.
\]
Since \(\partial_x b_n \to \partial_x b\) and \(\partial_x \sigma_n \to \partial_x \sigma\) uniformly with uniform bounds, we obtain
\[
\mathbb{E}\left[\sup_{t \in [0,T]} \|Y_t^n - Y_t\|^2\right] \to 0, \quad \mathbb{E}\left[\sup_{t \in [0,T]} \|(Y_t^n)^{-1} - Y_t^{-1}\|^2\right] \to 0.
\]
Decomposing the difference
\begin{align*}
D_t X_T^n - D_t X_T &= Y_T^n [(Y_t^n)^{-1} \sigma_n(t, X_t^n) - Y_t^{-1} \sigma(t, X_t)] \mathbf{1}_{\{t \leq T\}} \\
&\quad + (Y_T^n - Y_T) Y_t^{-1} \sigma(t, X_t) \mathbf{1}_{\{t \leq T\}}.
\end{align*}
Therefore
\begin{align*}
|D_t X_T^n - D_t X_T|^2 &\leq 2\|Y_T^n\|^2 \|(Y_t^n)^{-1} \sigma_n(t, X_t^n) - Y_t^{-1} \sigma(t, X_t)\|^2 \\
&\quad + 2\|Y_T^n - Y_T\|^2 \|Y_t^{-1} \sigma(t, X_t)\|^2.
\end{align*}
Setting \(I_1^n\) and \(I_2^n\) as the integrals of these terms, and using Cauchy-Schwarz with \(\sup_n \mathbb{E}[\|Y_T^n\|^4] < \infty\)
\[
I_2^n \leq 2(\mathbb{E}[\|Y_T^n - Y_T\|^4])^{1/2} \left(\mathbb{E}\left[\left(\int_0^T \|Y_t^{-1} \sigma(t, X_t)\|^2 dt\right)^2\right]\right)^{1/2} \to 0.
\]
Similarly, \(I_1^n \to 0\) since the integrand converges to zero and the first factor is bounded. Thus
\[
\mathbb{E}\left[\int_0^T |D_t X_T^n - D_t X_T|^2 dt\right] \to 0.
\]
%
We have
\[
\gamma^n_{ij} - \gamma_{ij} = \int_0^T (D_t X_T^{n,i} D_t X_T^{n,j} - D_t X_T^i D_t X_T^j) dt.
\]
Decomposing the product difference
\[
\gamma^n_{ij} - \gamma_{ij} = \int_0^T \left[(D_t X_T^{n,i} - D_t X_T^i)D_t X_T^{n,j} + D_t X_T^i(D_t X_T^{n,j} - D_t X_T^j)\right] dt,
\]
we obtain
\begin{align*}
\mathbb{E}\|\gamma^n - \gamma\|^2 \leq C \sum_{i,j} \bigg( &\mathbb{E}\left[\int_0^T |D_t X_T^{n,i} - D_t X_T^i|^2 |D_t X_T^{n,j}|^2 dt\right] \\
&+ \mathbb{E}\left[\int_0^T |D_t X_T^{n,j} - D_t X_T^j|^2 |D_t X_T^i|^2 dt\right] \bigg).
\end{align*}
We had, \(D_t X_T^n \to D_t X_T\) in \(L^2([0,T] \times \Omega)\). Combined with the uniform \(L^4\) bounds for \(D_t X_T^n\) (which hold under the \(C^2\) assumptions via \(D_t X_T^n = Y_T^n(Y_t^n)^{-1}\sigma_n(t,X_t^n)\)), the right-hand side converges to 0.
Since \(\gamma^n, \gamma \geq \lambda I_m\), we have the operator norm bound
\[
\|(\gamma^n)^{-1} - \gamma^{-1}\| \leq \lambda^{-2} \|\gamma^n - \gamma\|,
\]
which implies \(\mathbb{E}[\|(\gamma^n)^{-1} - \gamma^{-1}\|^2] \leq \lambda^{-4} \mathbb{E}[\|\gamma^n - \gamma\|^2] \to 0\).
%
We decompose
\begin{align*}
u_k^n(t) - u_k(t) &= \sum_{j=1}^m (\gamma^n)^{-1}_{jk} (D_t X_T^{n,j} - D_t X_T^j) \\
&\quad + \sum_{j=1}^m [(\gamma^n)^{-1}_{jk} - \gamma^{-1}_{jk}] D_t X_T^j.
\end{align*}
Squaring and using \(\sum_j |(\gamma^n)^{-1}_{jk}|^2 \leq \|(\gamma^n)^{-1}\|^2 \leq \lambda^{-2}\)
\begin{align*}
|u_k^n(t) - u_k(t)|^2 &\leq 2\lambda^{-2} \sum_j |D_t X_T^{n,j} - D_t X_T^j|^2 \\
&\quad + 2\|(\gamma^n)^{-1} - \gamma^{-1}\|^2 \sum_j |D_t X_T^j|^2.
\end{align*}
The integrand converges to 0 in \(L^1([0,T] \times \Omega)\) by Steps 2 and 3, and is dominated by an integrable function since \(\mathrm{Tr}(\gamma) \in L^2(\Omega)\). By dominated convergence
\[
\mathbb{E}\left[\int_0^T |u_k^n(t) - u_k(t)|^2 dt\right] \to 0. 
\]

\end{proof}

\section{Variation processes}
\label{sec:variation}

Given a stochastic process $X_t$ evolving from the initial condition $x$ in $\mathbb{R}^m$ we define the first variation process as \(Y_t = \partial X_t/\partial x \in \R^{m \times m}\). As is well known, $Y_t$ represents the sensitivity of \(X_t\) relative to small variations in the initial condition \(x\). Differentiating the SDE \eqref{sde} with respect to \(x\) yields
\[
dY_t = \partial_x b(t, X_t) Y_t \, dt + \sum_{l=1}^d \partial_x \sigma^l(t, X_t) Y_t \, dB_t^l, \quad Y_0 = I_m,
\]
where \(\partial_x b(t, X_t) \in \R^{m \times m}\), with \(\left[ \partial_x b(t, X_t) \right]_{i,j} = {\partial b^i(t, X_t)}/{\partial x_j}\), \(\partial_x \sigma^l(t, X_t) \in \R^{m \times m}\), with \(\left[ \partial_x \sigma^l(t, X_t) \right]_{i,j} = {\sigma^{i,l}(t, X_t)}/{\partial x_j}\), and \(I_m\) is the \(m \times m\) identity matrix.
The Malliavin derivative \(D_t X_T\) is the response of \(X_T\) to a perturbation in the Brownian motion at time \(t\). For \(t \leq T\)

\[
D_t X_T^j = \left[ Y_T Y_t^{-1} \sigma(t, X_t) \right]^j.
\]
For \(t > T\), \(D_t X_T = 0\) (future perturbations do not affect \(X_T\)). This follows because \(D_t X_s = 0\) for \(s < t\), and \(D_t X_t = \sigma(t, X_t)\), with the perturbation propagating via \(Y_{T,t} = Y_T Y_t^{-1}\).

The second variation process, denoted \(Z_t = \partial^2 X_t/\partial x^2\), is a third-order tensor in \(\mathbb{R}^{m \times m \times m}\). It represents the second-order sensitivities of the state process \(X_t\) with respect to the initial condition \(x\). Each component of \(Z_t\), written as \(Z_t^{i,p,q}\), corresponds to the second partial derivative \(\partial^2 X_t^i/\partial x_p \partial x_q\), where \(i, p, q = 1, \ldots, m\). This process evolves according to the following SDE
\begin{align*}
dZ_t &= \left[ \partial_{xx} b(t, X_t) (Y_t \otimes Y_t) + \partial_x b(t, X_t) Z_t \right] dt \\
&\quad + \sum_{l=1}^d \left[ \partial_{xx} \sigma^l(t, X_t) (Y_t \otimes Y_t) + \partial_x \sigma^l(t, X_t) Z_t \right] dB_t^l,
\end{align*}
With the initial condition \( Z_0 = 0 \), the stochastic differential equation for the second variation process \( Z_t \) is characterised by a collection of terms that  describe the system's dynamics. The Hessian tensor of the drift coefficient \( b \), denoted \( \partial_{xx} b(t, X_t) \in \mathbb{R}^{m \times m \times m} \), has components defined as \( \left[ \partial_{xx} b \right]_{i,j,k} = \partial^2 b^i(t, X_t) /(\partial x_j \partial x_k)\), representing the second derivatives of the \( i \)-th component of \( b \) with respect to the spatial variables \( x_j \) and \( x_k \). Likewise, the Hessian tensor of the \( l \)-th column of the diffusion coefficient \( \sigma \), expressed as \( \partial_{xx} \sigma^l(t, X_t) \in \mathbb{R}^{m \times m \times m} \), accounts for higher-order effects in the stochastic terms. The first variation process \( Y_t \in \mathbb{R}^{m \times m} \) captures the first-order sensitivities of the state process \( X_t \) with respect to the initial condition \( x \). This leads to the tensor product \( Y_t \otimes Y_t \), a fourth-order tensor in \( \mathbb{R}^{m \times m \times m \times m} \), which interacts with the Hessian tensors through contraction, enriching the structure of the second-variation SDE. The Jacobian matrices \( \partial_x b(t, X_t) \in \mathbb{R}^{m \times m} \) and \( \partial_x \sigma^l(t, X_t) \in \mathbb{R}^{m \times m} \) of \( b \) and \( \sigma^l \), respectively, provide the first-order spatial dependencies of the drift and diffusion coefficients. Driving the stochastic nature of the system, \( B_t^l \) represents the \( l \)-th component of a \( d \)-dimensional Brownian motion, introducing randomness into the evolution of \( Z_t \).
To clarify the tensor contraction, the term \(\partial_{xx} b(t, X_t) (Y_t \otimes Y_t)\) for each component \(Z_t^{i,p,q}\) is
\[
\left[ \partial_{xx} b(t, X_t) (Y_t \otimes Y_t) \right]^{i,p,q} = \sum_{j,k=1}^m \frac{\partial^2 b^i(t, X_t)}{\partial x_j \partial x_k} Y_t^{j,p} Y_t^{k,q}.
\]
Similarly, the term \(\partial_x b(t, X_t) Z_t\) is
\[
\left[ \partial_x b(t, X_t) Z_t \right]^{i,p,q} = \sum_{r=1}^m \frac{\partial b^i(t, X_t) }{\partial x_r}Z_t^{r,p,q}.
\]
The diffusion terms follow the same structure with \(\sigma^l\) replacing \(b\). By the classical existence‐and‐uniqueness theory for Itô–SDEs, together with the fact that all first and second‐order derivatives of \(b\) and \(\sigma\) are uniformly bounded (so that the coefficient‐processes in the linear equation are globally Lipschitz and square‐integrable), one obtains a unique strong solution \(Z_t\) (see, e.g., Kunita \cite{kunita1997stochastic}, Ch. 4).  


\section{Proof of Theorem \ref{thm:main}}
\label{sec:proofmain}
In this section we prove our main result, i.e., Theorem \ref{thm:main}. The proof relies on the results we  obtained in Section \ref{sec:malliavin} and Section \ref{sec:variation} for covering vector fields and variation processes to define the gradient of the log density for general stochastic differential equations using a Bismut-type formula. This formulation enables the computation of score functions for general nonlinear diffusion processes governed by stochastic differential equations. In this section, we first establish several preliminary results before proceeding to the proof of our main Theorem \ref{thm:main}.

We begin by considering the representation of the Skorokhod integral $\delta(u_k)$.  To this 	end, we first refer to Theorem 3.2.9 in \cite{nualart2006malliavin}, which provides a decomposition of the Skorokhod integral of a random field composed with a random variable. We also refer to the seminal works \cite{nualart1988stochastic, nualart1986generalized}, which study generalised stochastic integrals and anticipating integrals in Malliavin calculus. The result arises in the context of substitution formulae for stochastic integrals. Consider a random field \( u = \{ u_t(x) : 0 \leq t \leq T, x \in \mathbb{R}^m \} \) with \( u_t(x) \in \mathbb{R}^d \), which is square integrable and adapted for each \( x \in \mathbb{R}^m \). For each \( x \), one can define the It\^o integral \( \int_0^T u_t(x) \cdot dB_t \). Given an \( m \)-dimensional random variable \( F: \Omega \to \mathbb{R}^m \), Theorem 3.2.9 addresses the Skorokhod integrability of the nonadapted process \( u(F) = \{ u_t(F), 0 \leq t \leq T \} \) and provides a formula for its Skorokhod integral under the following conditions

\begin{itemize}
    \item[(h1)] For each \( x \in \mathbb{R}^m \) and \( t \in [0,T] \), \( u_t(x) \) is \( \mathcal{F}_t \)-measurable.
    
    \item[(h2)] There exist constants \( p \geq 2 \) and \( \alpha > m \) such that
    \[
    E(|u_t(x) - u_t(y)|^p) \leq C_{t,K} |x - y|^{\alpha},
    \]
    for all \( |x|, |y| \leq K \), \( K > 0 \), where \( \int_0^T C_{t,K} dt < \infty \). Moreover,
    \[
    \int_0^T E(|u_t(0)|^2) dt < \infty.
    \]
    \item[(h3)] For each \( (t, \omega) \), the mapping \( x \mapsto u_t(x) \) is continuously differentiable, and for each \( K > 0 \),
    \[
    \int_0^T E \left( \sup_{|x| \leq K} |\nabla u_t(x)|^q \right) dt < \infty,
    \]
    where \( q \geq 4 \) and \( q > m \).
\end{itemize}

For the reader’s convenience, we recall Theorem 3.2.9 from \cite{nualart2006malliavin} below.

\begin{theorem}[Theorem 3.2.9, \cite{nualart2006malliavin}]
\label{thm:nualart}
For a random field \( u = \{ u_t(x) : 0 \leq t \leq T, x \in \mathbb{R}^m \} \) with \( u_t(x) \in \mathbb{R}^d \), and a random variable \( F: \Omega \to \mathbb{R}^m \) such that \( F^i \in \mathbb{D}^{1,4}_{\text{loc}} \) for \( 1 \leq i \leq m \), assume \( u \) satisfies the conditions (h1) and (h3) for Skorokhod integrability.
Then, the composition \( u(F) = \{ u_t(F), 0 \leq t \leq T \} \) belongs to \( (\text{Dom} (\delta))_{\text{loc}} \), and the Skorokhod integral of \( u(F) \) is given by
\[
\delta(u(F)) = \left. \int_0^T u_t(x) \cdot dB_t \right|_{x=F} - \sum_{j=1}^m \int_0^T \partial_j u_t(F) \cdot D_t F^j \, dt,
\]
where \( B_t \) is a \( d \)-dimensional Brownian motion, \( \partial_j u_t(x) = \partial u_t(x)/\partial x_j  \) is the partial derivative of \( u_t(x) \) with respect to the \( j \)-th component of \( x \), \( D_t F^j \) is the Malliavin derivative of the \( j \)-th component of \( F \), and \( \left. \int_0^T u_t(x) \cdot dB_t \right|_{x=F} \) denotes the It\^o integral evaluated at \( x = F \).
We note that no smoothness in the sense of Malliavin calculus is required on the process \( u_t(x) \) itself, but the above conditions ensure the integrability of \( u(F) \) in the Skorokhod sense. Furthermore, the operator \( \delta \) is not known to be local in \( \text{Dom} (\delta) \), and thus the value of \( \delta(u(F)) \) may depend on the particular localising sequence used in the definition of \( (\text{Dom}( \delta))_{\text{loc}} \).
\end{theorem}


\subsection{Useful lemmas}

In this section, we state and prove some key results that are useful for deriving the 
score function formula \eqref{score}-\eqref{score1} for nonlinear SDEs 
in terms of the first and second variation processes.
\begin{lemma}[SDE for the inverse first variation process]
\label{lemma:dYt-inverse}
Let $Y_t$ be the first variation process. The inverse \( Y_t^{-1} \) satisfies the SDE
\begin{align*}
dY_t^{-1} &= - Y_t^{-1} \partial_x b(t, X_t) \, dt - \sum_{l=1}^d Y_t^{-1} \partial_x \sigma^l(t, X_t) \, dB_t^l + \sum_{l=1}^d Y_t^{-1} \left( \partial_x \sigma^l(t, X_t) \right)^2 \, dt,
\end{align*}
with initial condition \( Y_0^{-1} = I_m \), where \( \left( \partial_x \sigma^l(t, X_t) \right)^2 = \partial_x \sigma^l(t, X_t) \partial_x \sigma^l(t, X_t) \).
\end{lemma}

\begin{proof}
Since \( Y_t Y_t^{-1} = I_m \) is constant, its differential is zero, i.e.,
\begin{align*}
d(Y_t Y_t^{-1}) &= dY_t Y_t^{-1} + Y_t dY_t^{-1} + d[Y_t, Y_t^{-1}] = 0.
\end{align*}
Given the first-variation SDE
\begin{align}
dY_t &= \partial_x b(t, X_t) Y_t \, dt + \sum_{l=1}^d \partial_x \sigma^l(t, X_t) Y_t \, dB_t^l,
\label{fvSDE}
\end{align}
we compute
\begin{align*}
dY_t Y_t^{-1} &= \left( \partial_x b(t, X_t) Y_t \, dt + \sum_{l=1}^d \partial_x \sigma^l(t, X_t) Y_t \, dB_t^l \right) Y_t^{-1} \\
&= \partial_x b(t, X_t) \, dt + \sum_{l=1}^d \partial_x \sigma^l(t, X_t) \, dB_t^l,
\end{align*}
where we used \( Y_t Y_t^{-1} = I_m \). We assume that
\begin{align*}
dY_t^{-1} &= \mu_t \, dt + \sum_{l=1}^d \nu_t^l \, dB_t^l.
\end{align*}
We need to show \( Y_t^{-1} \) is an It\^o process. To this end, consider the SDE \eqref{fvSDE}
with \( Y_0 = I_m \). The coefficients \( \partial_x b(t, X_t) \) and \( \partial_x \sigma^l(t, X_t) \) are assumed to be bounded and measurable. The SDE  \eqref{fvSDE} is linear, and under standard conditions (e.g., Lipschitz continuity of the coefficients), the solution \( Y_t \) exists, is unique, and is a semimartingale, specifically, an It\^o process, adapted to the filtration generated by the Brownian motions \( B_t^l \). Moreover, since \( Y_0 = I_m \) is invertible and the coefficients satisfy regularity conditions, \( Y_t \) remains invertible almost surely for all \( t \geq 0 \). Define the function \( f: \text{GL}(m, \mathbb{R}) \to \text{GL}(m, \mathbb{R}) \) by \( f(A) = A^{-1} \), where \( \text{GL}(m, \mathbb{R}) \) is the group of \( m \times m \) invertible matrices. The map \( f \) is smooth (infinitely differentiable) on \( \text{GL}(m, \mathbb{R}) \), with first derivative \( Df(A)H = -A^{-1} H A^{-1} \) and second derivative terms involving higher-order products. Applying It\^o' s formula to \( Y_t^{-1} = f(Y_t) \), where \( Y_t \) is an It\^o process, yields a stochastic differential of the form
\begin{align*}
dY_t^{-1} &= \mu_t \, dt + \sum_{l=1}^d \nu_t^l \, dB_t^l,
\end{align*}
where \( \mu_t \) and \( \nu_t^l \) are adapted processes derived from the drift and diffusion terms of \( Y_t \). Thus, \( Y_t^{-1} \) is itself an Itô process, and we proceed to determine \( \mu_t \) and \( \nu_t^l \).
Continuing with the proof
\begin{align*}
Y_t dY_t^{-1} &= Y_t \mu_t \, dt + \sum_{l=1}^d Y_t \nu_t^l \, dB_t^l.
\end{align*}
The quadratic covariation is
\begin{align*}
d[Y_t, Y_t^{-1}] &= \sum_{l=1}^d \left( \partial_x \sigma^l(t, X_t) Y_t \right) \nu_t^l \, dt.
\end{align*}
Substituting into the differential
\begin{align*}
d(Y_t Y_t^{-1}) &= \left( \partial_x b(t, X_t) + Y_t \mu_t + \sum_{l=1}^d \partial_x \sigma^l(t, X_t) Y_t \nu_t^l \right) dt \\
&\quad + \sum_{l=1}^d \left( \partial_x \sigma^l(t, X_t) + Y_t \nu_t^l \right) dB_t^l \\
&= 0,
\end{align*}
and equating the coefficients yields
\begin{align*}
&dB_t^l\text{-term}: Y_t \nu_t^l + \partial_x \sigma^l(t, X_t) = 0, \quad \Rightarrow\quad  \nu_t^l = - Y_t^{-1} \partial_x \sigma^l(t, X_t), \\
&dt\text{-term}: \quad Y_t \mu_t + \partial_x b(t, X_t) + \sum_{l=1}^d \partial_x \sigma^l(t, X_t) Y_t (- Y_t^{-1} \partial_x \sigma^l(t, X_t)) = 0.
\end{align*}
Simplifying 
\begin{align*}
\sum_{l=1}^d \partial_x \sigma^l(t, X_t) Y_t (- Y_t^{-1} \partial_x \sigma^l(t, X_t)) &= - \sum_{l=1}^d \partial_x \sigma^l(t, X_t) \partial_x \sigma^l(t, X_t),
\end{align*}
\begin{align*}
Y_t \mu_t + \partial_x b(t, X_t) - \sum_{l=1}^d \partial_x \sigma^l(t, X_t) \partial_x \sigma^l(t, X_t) &= 0,
\end{align*}
\begin{align*}
\mu_t &= Y_t^{-1} \left( - \partial_x b(t, X_t) + \sum_{l=1}^d \left( \partial_x \sigma^l(t, X_t) \right)^2 \right).
\end{align*}
Thus
\begin{align*}
dY_t^{-1} &= Y_t^{-1} \left( - \partial_x b(t, X_t) + \sum_{l=1}^d \left( \partial_x \sigma^l(t, X_t) \right)^2 \right) dt - \sum_{l=1}^d Y_t^{-1} \partial_x \sigma^l(t, X_t) \, dB_t^l.
\end{align*}
The initial condition \( Y_0^{-1} = I_m \) holds since \( Y_0 = I_m \).

\end{proof}

\begin{lemma}[Malliavin derivative of inverse matrices]
\label{lemma:D-of-A-inverse}
Let $A$ be a random $m \times m$ matrix that is invertible almost surely, and assume there exists $\lambda > 0$ such that $A \geq \lambda I_m$ almost surely (uniform ellipticity). Suppose further that $A$ and $A^{-1}$ are Malliavin differentiable. Then, for each $t \in [0,T]$,
\[
D_t(A^{-1}) = -A^{-1}(D_tA)A^{-1}.
\]
\end{lemma}

\begin{proof}
Since $AA^{-1} = I$, applying the Malliavin derivative yields
\[
D_t(AA^{-1}) = (D_tA)A^{-1} + A(D_tA^{-1}) = D_tI = 0.
\]
Consequently,
\[
A(D_tA^{-1}) = -(D_tA)A^{-1}.
\]
Multiplying on the left by $A^{-1}$ gives the desired result.

\end{proof}

\begin{remark}[The role of integrability and ellipticity conditions]
\label{rem:ellipticity-vs-moments}
While Lemma~\ref{lemma:D-of-A-inverse} is purely algebraic, ensuring that $A^{-1}$ (and the right-hand side) belongs to the Malliavin--Sobolev space $\mathbb{D}^{1,2}$ requires controlling the integrability of $\|A^{-1}\|$.
If $A \geq \lambda I_m$ almost surely for some deterministic $\lambda > 0$, then $\|A^{-1}\| \leq \lambda^{-1}$ and all polynomial moments of $A^{-1}$ are finite. Under this bound, every entry of $A^{-1}$ lies in $\mathbb{D}^\infty$, and the componentwise identity
\[
D(A^{-1})^{ij} = -\sum_{k,\ell=1}^{m}(A^{-1})^{ik}(A^{-1})^{\ell j}\,D A^{k\ell},
\]
established in \cite[Lemma 2.1.6]{nualart2006malliavin}, is justified.
Without such a lower bound on the eigenvalues, the inverse may fail to be integrable. In dimension one, for instance, $A = \int_0^{T}B_s^{2}\,ds > 0$ satisfies $A^{-q} \in L^{1}(\Omega)$ only for $q < \frac{1}{2}$; thus $A^{-1} \notin \mathbb{D}^{1,2}$ and Lemma 2.1.6 cannot be applied. Uniform ellipticity (or, more generally, the moment hypothesis $|\det A|^{-1} \in L^{p}$ for all $p \geq 1$ required in \cite[Lemma 2.1.6]{nualart2006malliavin}) prevents this pathology.
\end{remark}

\begin{lemma}[Commutativity of Malliavin and partial derivatives]
\label{lemma:commutativity-malliavin-partials}
Let $X_t = X_t(x)$ be the solution to the stochastic differential equation
\[
\mathrm{d}X_t = b(t,X_t)\,\mathrm{d}t + \sigma(t,X_t)\,\mathrm{d}B_t,
\quad X_0 = x, \quad 0 \leq t \leq T,
\]
where $B_t$ is a $d$-dimensional standard Brownian motion and $x \in \mathbb{R}^{m}$.
Assume the coefficients satisfy
$b, \sigma \in \mathcal{C}^{2}_{\!b}([0,T] \times \mathbb{R}^{m})$
(all partial derivatives up to second order are continuous and bounded).
Let $Y_t = \partial X_t/\partial x$ denote the first variation process.
Then, for every $0 \leq t \leq T$, where $D_t$ denotes the Malliavin derivative,
\[
D_t\left(\frac{\partial X_T}{\partial x}\right)
=
\frac{\partial}{\partial x}(D_t X_T).
\]
\end{lemma}

\begin{proof}
We compute both sides explicitly to verify equality. The SDE for $X_t$ can be rewritten in the integral form
\begin{align*}
X_t &= x + \int_0^t b(s, X_s) \, ds + \int_0^t \sigma(s, X_s) \, dB_s.
\end{align*}
The first variation process \(Y_t = \partial X_t/\partial x\) satisfies
\begin{align*}
dY_t &= \partial_x b(t, X_t) Y_t \, \dt + \sum_{l=1}^d \partial_x \sigma^l(t, X_t) Y_t \, dB_t^l, \quad Y_0 = I_m.
\end{align*}
We have seen that the inverse \(Y_t^{-1}\) satisfies the SDE 
\begin{align*}
dY_t^{-1} &= - Y_t^{-1} \partial_x b(t, X_t) \, dt \\
&\quad - \sum_{l=1}^d Y_t^{-1} \partial_x \sigma^l(t, X_t) \, dB_t^l \\
&\quad + \sum_{l=1}^d Y_t^{-1} \left( \partial_x \sigma^l(t, X_t) \right)^2 \, dt, \quad Y_0^{-1} = I_m
\end{align*}
while the second variation process satisfies
\begin{align*}
dZ_t &= \left[ \partial_{xx} b(t, X_t) (Y_t \otimes Y_t) + \partial_x b(t, X_t) Z_t \right] \dt \\
&\quad + \sum_{l=1}^d \left[ \partial_{xx} \sigma^l(t, X_t) (Y_t \otimes Y_t) + \partial_x \sigma^l(t, X_t) Z_t \right] dB_t^l, \quad Z_0 = 0.
\end{align*}
\begin{align*}
D_t X_T &= Y_T Y_t^{-1} \sigma(t, X_t).
\end{align*}
Then
\begin{align*}
\frac{\partial}{\partial x} (D_t X_T) &= \frac{\partial}{\partial x} \left( Y_T Y_t^{-1} \sigma(t, X_t) \right) \\
&= \frac{\partial Y_T}{\partial x} \left( Y_t^{-1} \sigma(t, X_t) \right) + Y_T \frac{\partial}{\partial x} \left( Y_t^{-1} \sigma(t, X_t) \right).
\end{align*}
Since \(\partial Y_T/\partial x = Z_T\) we have
\begin{align*}
\frac{\partial Y_T}{\partial x} \left( Y_t^{-1} \sigma(t, X_t) \right) &= Z_T Y_t^{-1} \sigma(t, X_t).
\end{align*}
For the second term
\begin{align*}
\frac{\partial}{\partial x} \left( Y_t^{-1} \sigma(t, X_t) \right) &= \frac{\partial Y_t^{-1}}{\partial x} \sigma(t, X_t) + Y_t^{-1} \frac{\partial \sigma(t, X_t)}{\partial x},
\end{align*}
where \(\partial Y_t^{-1}/\partial x = - Y_t^{-1} Z_t Y_t^{-1}\) and \(\partial \sigma(t, X_t)/\partial x = \partial_x \sigma(t, X_t) Y_t\). Thus
\begin{align*}
\frac{\partial}{\partial x} \left( Y_t^{-1} \sigma(t, X_t) \right) &= - \left( Y_t^{-1} Z_t Y_t^{-1} \right) \sigma(t, X_t) + Y_t^{-1} \left( \partial_x \sigma(t, X_t) Y_t \right),
\end{align*}
\begin{align*}
\frac{\partial}{\partial x} (D_t X_T) &= Z_T Y_t^{-1} \sigma(t, X_t) + Y_T \left( - \left( Y_t^{-1} Z_t Y_t^{-1} \right) \sigma(t, X_t) + Y_t^{-1} \left( \partial_x \sigma(t, X_t) Y_t \right) \right).
\end{align*}
Moreover, since \(\partial X_T/\partial x = Y_T\) we have
\begin{align*}
Y_T &= I_m + \int_0^T \partial_x b(s, X_s) Y_s \, ds + \int_0^T \partial_x \sigma(s, X_s) Y_s \, dB_s,
\end{align*}
\begin{align*}
D_t Y_T &= \partial_x \sigma(t, X_t) Y_t + \int_t^T \left[ \partial_x b(s, X_s) D_t Y_s + \partial_{xx} b(s, X_s) \left( Y_s Y_t^{-1} \sigma(t, X_t) \right) Y_s \right] ds \\
&\quad + \int_t^T \left[ \partial_x \sigma(s, X_s) D_t Y_s + \partial_{xx} \sigma(s, X_s) \left( Y_s Y_t^{-1} \sigma(t, X_t) \right) Y_s \right] dB_s.
\end{align*}
The solution is
\begin{align*}
D_t Y_T &= Z_T Y_t^{-1} \sigma(t, X_t) - Y_T Y_t^{-1} Z_t Y_t^{-1} \sigma(t, X_t) + Y_T Y_t^{-1} \partial_x \sigma(t, X_t) Y_t.
\end{align*}
Clearly the expressions are equal, confirming that
\begin{align*}
D_t \left( \frac{\partial X_T}{\partial x} \right) &= \frac{\partial}{\partial x} \left( D_t X_T \right).
\end{align*}
The regularity conditions ensure all derivatives and integrals are well-defined.

\end{proof}

\begin{lemma}[Malliavin derivative for the first variation process]
\label{lemma:DtYT}
For \( t \leq T \), the Malliavin derivative of the first variation process \( Y_T \) is given by
\begin{align*}
D_t Y_T = &\,Z_T\,Y_t^{-1}\,\sigma(t, X_t) \\
&-\, Y_T\,Y_t^{-1}\,Z_t\,Y_t^{-1}\,\sigma(t, X_t) \\
&+\, Y_T\,Y_t^{-1}\,\partial_x \sigma(t, X_t)\,Y_t.
\end{align*}
\end{lemma}

\begin{proof}
\noindent
Consider the SDE satisfied by the first variation process 
\begin{align*}
dY_s 
= \partial_x b(s, X_s)\, Y_s \, ds 
   \;+\;
   \partial_x \sigma(s, X_s)\, Y_s \, dB_s, \qquad 
Y_0 = I.
\end{align*}
Recall that \(Y_s =\partial X_s/\partial x\). 
The Malliavin derivative \( D_t Y_T \) represents the sensitivity of \( Y_T \) 
to a perturbation in the Brownian motion at time \( t \). Since 
\( Y_T = \partial X_T/\partial x \) we have
\begin{align*}
D_t Y_T 
&= D_t \left(\frac{\partial X_T}{\partial x}\right) 
= \frac{\partial}{\partial x} \left(D_t X_T\right).
\end{align*}
This identity follows from interchanging the partial derivative w.r.t.\ \(x\) and the Malliavin derivative \(D_t\). We use the known expression 
\(\displaystyle D_t X_T = Y_T\,\bigl(Y_t^{-1}\,\sigma(t, X_t)\bigr)\) 
which is valid for all \( t \leq T\).  Define \(\displaystyle W_t = Y_t^{-1}\,\sigma(t, X_t)\).  
Thus,
\begin{align*}
D_t X_T 
&= Y_T \,W_t.
\end{align*}
Differentiate this equation with respect to the initial condition \(x\) to obtain
\begin{align*}
\frac{\partial}{\partial x} (D_t X_T)
=\frac{\partial Y_T}{\partial x}W_t +\
   Y_T \,\frac{\partial W_t}{\partial x}.
\end{align*}
By definition of second variation process \(\partial Y_T/\partial x = Z_T\). Hence
\begin{align}
D_t Y_T= Z_T\,W_t + Y_T \frac{\partial W_t}{\partial x}.
\label{a1}
\end{align}
Recalling that \( W_t = Y_t^{-1}\,\sigma(t, X_t)\) we can write
\begin{align*}
\frac{\partial W_t}{\partial x}
&= \frac{\partial}{\partial x} 
   \left(Y_t^{-1} \,\sigma(t, X_t)\right) \\
&= \left(\frac{\partial Y_t^{-1}}{\partial x}\right)\,\sigma(t, X_t)+
   Y_t^{-1}\,\frac{\partial \sigma(t, X_t)}{\partial x}.
\end{align*}
Each of these two terms comes from the product rule (now for partial derivatives w.r.t.\ \(x\)).
Since \(Y_t\,Y_t^{-1} = I\), differentiating both sides w.r.t.\ \(x\) yields 
\(\partial Y_t/\partial x\,Y_t^{-1} + Y_t\,\partial Y_t^{-1}/\partial x = 0.\)  
Hence
\begin{align*}
\frac{\partial Y_t^{-1}}{\partial x}
&= -Y_t^{-1}\,\left(\frac{\partial Y_t}{\partial x}\right)\,Y_t^{-1}
= -Y_t^{-1}Z_tY_t^{-1}.
\end{align*}
We also have \(\partial X_t/\partial x = Y_t\). Thus by chain rule
\begin{align*}
\frac{\partial \sigma(t, X_t)}{\partial x}
&= \partial_x \sigma(t, X_t)\,Y_t.
\end{align*}
Therefore,
\begin{align*}
\frac{\partial W_t}{\partial x}
&= -Y_t^{-1}Z_tY_t^{-1}\sigma(t, X_t)+
   Y_t^{-1}\partial_x \sigma(t, X_t)Y_t.
\end{align*}
This completes the computation of \(\partial W_t/\partial x\). Putting this result back 
into \eqref{a1} yields
\begin{align*}
D_t Y_T
&= Z_T\,Y_t^{-1}\,\sigma(t, X_t)
   \;+\;
   Y_T\Bigl[
        -\,Y_t^{-1}\,Z_t\,Y_t^{-1}\,\sigma(t, X_t)
        \;+\;
        Y_t^{-1}\,\partial_x \sigma(t, X_t)\,Y_t
       \Bigr].
\end{align*}
Factor out common terms to rewrite it in the stated form
\begin{align*}
D_t Y_T = &\,Z_T\,Y_t^{-1}\,\sigma(t, X_t) -\, Y_T\,Y_t^{-1}\,Z_t\,Y_t^{-1}\,\sigma(t, X_t) 
+\, Y_T\,Y_t^{-1}\,\partial_x \sigma(t, X_t)\,Y_t.
\end{align*}
The second term \(\,-\,Y_T\,Y_t^{-1}\,Z_t\,Y_t^{-1}\,\sigma(t, X_t)\)  
accounts for the appearance of the second variation \(Z_t\) inside the inverse, 
while the final term 
\(\,Y_T\,Y_t^{-1}\,\partial_x \sigma(t, X_t)\,Y_t\)
encodes the effect of differentiating \(\sigma\) itself w.r.t.\ \(x\).

\end{proof}

\begin{lemma}[Malliavin derivative formula for the inverse first variation process]
\label{lemma:DtYs_inverse}
For the inverse first variation process \( Y_s^{-1} \), the Malliavin derivative is given by
\begin{itemize}
    \item For \( t \leq s \):
    \begin{align*}
    D_t Y_s^{-1} &= - Y_s^{-1} \left[ Z_s Y_t^{-1} \sigma(t, X_t) - Y_s Y_t^{-1} Z_t Y_t^{-1} \sigma(t, X_t) + Y_s Y_t^{-1} \partial_x \sigma(t, X_t) Y_t \right] Y_s^{-1}
    \end{align*}
    \item For \( t > s \):
    \[
    D_t Y_s^{-1} = 0
    \]
\end{itemize}
where \( Z_s = \partial^2 X_s/\partial x^2 \) is the second variation process.
\end{lemma}

\begin{proof}
We derive \( D_t Y_s^{-1} \) by applying the Malliavin derivative to the identity \( Y_s Y_s^{-1} = I_m \) and using the product rule. The proof splits into two cases based on the relationship between \( t \) and \( s \), and we leverage the expression for \( D_t Y_s \) derived similarly to \( D_t Y_T \).
Let us begin with the case \( t \leq s \). Since \( Y_s Y_s^{-1} = I_m \) (the \( m \times m \) identity matrix), we apply the Malliavin derivative \( D_t \) to both sides
\begin{align*}
D_t (Y_s Y_s^{-1}) &= D_t (I_m) = 0.
\end{align*}
Using the product rule for Malliavin derivatives
\begin{align*}
D_t (Y_s Y_s^{-1}) = (D_t Y_s) Y_s^{-1} + Y_s (D_t Y_s^{-1}) = 0.
\end{align*}
Rearranging to isolate \( D_t Y_s^{-1} \)
\begin{align*}
Y_s (D_t Y_s^{-1}) &= - (D_t Y_s) Y_s^{-1}, \\
D_t Y_s^{-1} &= - Y_s^{-1} (D_t Y_s) Y_s^{-1}.
\end{align*}
To proceed, we need \( D_t Y_s \). Since \( Y_s = \partial X_s/\partial x \) and \( t \leq s \), we adapt the derivation from the previous Lemma \ref{lemma:DtYT} for \( D_t Y_T \), adjusting the endpoint from \( T \) to \( s \)
\begin{align*}
dY_u = \partial_x b(u, X_u) Y_u \, du + \partial_x \sigma(u, X_u) Y_u \, dB_u, \quad 0 \leq u \leq s, \qquad 
Y_0 = I.
\end{align*}
We have 
\begin{align*}
D_t Y_s &= D_t \left( \frac{\partial X_s}{\partial x} \right) = \frac{\partial}{\partial x} (D_t X_s).
\end{align*}
For \( t \leq s \), the Malliavin derivative of \( X_s \) is
\begin{align*}
D_t X_s &= Y_s Y_t^{-1} \sigma(t, X_t).
\end{align*}
Define \( W_t = Y_t^{-1} \sigma(t, X_t) \), so
\begin{align*}
D_t X_s &= Y_s W_t.
\end{align*}
\begin{align*}
\frac{\partial}{\partial x} (D_t X_s) &= \frac{\partial}{\partial x} (Y_s W_t) = \frac{\partial Y_s}{\partial x} W_t + Y_s \frac{\partial W_t}{\partial x}.
\end{align*}
Since \( \partial Y_s/\partial x = Z_s \) (second variation process)
\begin{align*}
D_t Y_s &= Z_s W_t + Y_s \frac{\partial W_t}{\partial x}.
\end{align*}
Recalling that $W_t = Y_t^{-1} \sigma(t, X_t)$, 
\begin{align*}
\frac{\partial W_t}{\partial x} &= \frac{\partial Y_t^{-1}}{\partial x} \sigma(t, X_t) + Y_t^{-1} \frac{\partial \sigma(t, X_t)}{\partial x}.
\end{align*}
For \( \partial Y_t^{-1}/\partial x \) differentiate \( Y_t Y_t^{-1} = I \) to obtain
\begin{align*}
\frac{\partial Y_t^{-1}}{\partial x} &= - Y_t^{-1} \frac{\partial Y_t}{\partial x} Y_t^{-1} = - Y_t^{-1} Z_t Y_t^{-1}.
\end{align*}
For \( \partial \sigma(t, X_t)/\partial x\), since \( \partial X_t/\partial x = Y_t \), we have
\begin{align*}
\frac{\partial \sigma(t, X_t)}{\partial x} &= \partial_x \sigma(t, X_t) Y_t.
\end{align*}
Therefore
\begin{align*}
\frac{\partial W_t}{\partial x} &= - Y_t^{-1} Z_t Y_t^{-1} \sigma(t, X_t) + Y_t^{-1} \partial_x \sigma(t, X_t) Y_t
\end{align*}
and
\begin{align*}
D_t Y_s &= Z_s Y_t^{-1} \sigma(t, X_t) + Y_s \left( - Y_t^{-1} Z_t Y_t^{-1} \sigma(t, X_t) + Y_t^{-1} \partial_x \sigma(t, X_t) Y_t \right), \\
&= Z_s Y_t^{-1} \sigma(t, X_t) - Y_s Y_t^{-1} Z_t Y_t^{-1} \sigma(t, X_t) + Y_s Y_t^{-1} \partial_x \sigma(t, X_t) Y_t.
\end{align*}
Finally,  substitute \( D_t Y_s \) into \( D_t Y_s^{-1} \) to obtain
\begin{align*}
D_t Y_s^{-1} &= - Y_s^{-1} \left[ Z_s Y_t^{-1} \sigma(t, X_t) - Y_s Y_t^{-1} Z_t Y_t^{-1} \sigma(t, X_t) + Y_s Y_t^{-1} \partial_x \sigma(t, X_t) Y_t \right] Y_s^{-1}.
\end{align*}
Since \( Y_s^{-1} \) is adapted to the filtration up to time \( s \), and \( t > s \), a perturbation in the Brownian motion at time \( t \) does not affect \( Y_s^{-1} \) (which depends only on information up to \( s \)). Thus
\begin{align*}
D_t Y_s^{-1} &= 0.
\end{align*}
This completes the proof, with the expression for \( t \leq s \) matching the Lemma statement, and the zero result for \( t > s \) reflecting the causality of the stochastic process.

\end{proof}

\subsection{Proof of Theorem \ref{thm:main}}
Having established all necessary preliminary results, we now proceed to prove our Theorem \ref{thm:main}.

\begin{proof}
Applying Theorem \ref{thm:nualart}, the Skorokhod 
integral \( \delta(u_k) = \delta(u(F_k)) \) is
\begin{align}
\delta(u_k) = \left. \int_0^T u_t(x) \cdot dB_t \right|_{x=F_k} - \sum_{j=1}^m \int_0^T \partial_j u_t(F_k) \cdot D_t F_k^j \, dt.
\label{Snualart}
\end{align}
This expression combines an It\^o integral evaluated at \( x = F_k \) with a correction term involving the partial derivatives of \( u_t(x) \) evaluated at \( x = F_k \) 
and the Malliavin derivatives of \( F_k^j \).
The first term in the expression for \( \delta(u(F_k)) \) is the Itô integral evaluated at \( x = F_k \)
\begin{align*}
\left. \int_0^T u_t(x) \cdot dB_t \right|_{x=F_k},
\end{align*}
where for each fixed \( x \), \( u_t(x) = x^\top Y_t^{-1} \sigma(t, X_t) \) is an adapted process, so \( \int_0^T u_t(x) \cdot dB_t \) is a well-defined Itô integral, and after computing this integral, we evaluate it at \( x = F_k = Y_T^\top \gamma_{X_T}^{-1} e_k \), which is \( \mathcal{F}_T \)-measurable.
This term is computationally manageable because the integration is performed with respect to an adapted integrand for fixed \( x \), and the randomness of \( F_k \) is introduced only after the integration.
With the redefined random field
\begin{align*}
u_t(x) &= x^\top Y_t^{-1} \sigma(t, X_t) = \sum_{i=1}^m x_i \left[ Y_t^{-1} \sigma(t, X_t) \right]_i,
\end{align*}
the partial derivative with respect to \( x_j \) is
\begin{align*}
\partial_j u_t(x) &= \frac{\partial}{\partial x_j} u_t(x) = \left[ Y_t^{-1} \sigma(t, X_t) \right]_j,
\end{align*}
since only the term involving \( x_j \) depends on \( x_j \). Therefore, evaluating at \( x = F_k \)
\begin{align*}
\partial_j u_t(F_k) &= \left[ Y_t^{-1} \sigma(t, X_t) \right]_j.
\end{align*}
This term will appear in the correction term of the Skorokhod integral decomposition.
Before proceeding further, we state a general result from Lemma \ref{lemma:D-of-A-inverse} for the Malliavin derivative of the inverse of a random matrix.
To compute the Malliavin derivative \( D_t F_k^j \), consider the new definition
\begin{align*}
F_k = Y_T^\top \gamma_{X_T}^{-1} e_k, \qquad 
F_k^j = e_j^\top Y_T^\top \gamma_{X_T}^{-1} e_k,
\end{align*}
where \( Y_T \) is the first variation process at time \( T \), \( \gamma_{X_T} \) is the Malliavin covariance matrix, and \( e_j, e_k \) are standard basis vectors. Applying the Malliavin derivative
\begin{align*}
D_t F_k^j &= e_j^\top D_t (Y_T^\top \gamma_{X_T}^{-1}) e_k.
\end{align*}
Using the product rule, we obtain
\begin{align*}
D_t (Y_T^\top \gamma_{X_T}^{-1}) &= (D_t Y_T^\top) \gamma_{X_T}^{-1} + Y_T^\top D_t (\gamma_{X_T}^{-1}), \\
D_t F_k^j &= e_j^\top (D_t Y_T^\top) \gamma_{X_T}^{-1} e_k + e_j^\top Y_T^\top D_t (\gamma_{X_T}^{-1}) e_k.
\end{align*}
%
For \( t \leq T \), the Malliavin derivative \( D_t Y_T \) is given in Lemma \ref{lemma:DtYT} 
\begin{align*}
D_t Y_T &= Z_T Y_t^{-1} \sigma(t, X_t) - Y_T Y_t^{-1} Z_t Y_t^{-1} \sigma(t, X_t) + Y_T Y_t^{-1} \partial_x \sigma(t, X_t) Y_t,
\end{align*}
where \( Z_t \) is the second variation process. Taking the transpose
\begin{align*}
D_t Y_T^\top &= \left[ Z_T Y_t^{-1} \sigma(t, X_t) - Y_T Y_t^{-1} Z_t Y_t^{-1} \sigma(t, X_t) + Y_T Y_t^{-1} \partial_x \sigma(t, X_t) Y_t \right]^\top.
\end{align*}
From  Lemma \ref{lemma:D-of-A-inverse} we have that for an invertible random matrix \( A \)
\begin{align*}
D_t (A^{-1}) &= - A^{-1} (D_t A) A^{-1}. 
\end{align*}
For the specific case of \( A = \gamma_{X_T} \) this yields
\begin{align*}
D_t (\gamma_{X_T}^{-1}) &= - \gamma_{X_T}^{-1} (D_t \gamma_{X_T}) \gamma_{X_T}^{-1},
\end{align*}
where
\begin{align*}
\gamma_{X_T} &= \int_0^T Y_T Y_s^{-1} \sigma(s, X_s) \sigma(s, X_s)^\top (Y_s^{-1})^\top Y_T^\top \, ds,
\end{align*}
and \( D_t \gamma_{X_T} \) requires computing the Malliavin derivative 
of the integrand. Thus,
\begin{align*}
D_t F_k^j &= e_j^\top (D_t Y_T^\top) \gamma_{X_T}^{-1} e_k + e_j^\top Y_T^\top D_t (\gamma_{X_T}^{-1}) e_k, \\
D_t F_k^j &= e_j^\top (D_t Y_T^\top) \gamma_{X_T}^{-1} e_k - e_j^\top Y_T^\top \gamma_{X_T}^{-1} (D_t \gamma_{X_T}) \gamma_{X_T}^{-1} e_k,
\end{align*}
with \( D_t Y_T^\top \) and \( D_t \gamma_{X_T} \) as derived.
We now proceed by computing the Malliavin derivative 
\(\displaystyle D_t \Bigl[(Y_T\,Y_s^{-1}\,\sigma(s,X_s))^p\Bigr]\). 
To this end, we first recall Lemma \ref{lemma:DtYT} which gives  the 
precise form of \(\displaystyle D_t Y_T\). Afterwards, we use the product 
rule for Malliavin derivatives on the product \(\,Y_T \,Y_s^{-1}\,\sigma(s,X_s)\), 
distinguishing between the cases \(t \le s\) and \(t > s\). 
Finally, we assemble these pieces to obtain the expression for 
\(\displaystyle D_t (Y_T Y_s^{-1} \sigma(s, X_s))^p\).  
We provide reasoning for each step to clarify why each term appears 
and how the partial derivatives interact with the inverse processes.
Let 
\[
W_s^p 
\;=\;
\Bigl(Y_T\,Y_s^{-1}\,\sigma(s, X_s)\Bigr)^p,
\]
i.e.\ the \(p\)-th component of the vector \(Y_T\,Y_s^{-1}\,\sigma(s,X_s)\).  
We want to find 
\(\displaystyle D_t \bigl(W_s^p\bigr)\).  
Since
\[
W_s^p 
= 
\bigl(Y_T\,Y_s^{-1}\,\sigma(s, X_s)\bigr)^p,
\]
we begin with the Malliavin derivative of the product 
\(\,Y_T\,Y_s^{-1}\,\sigma(s,X_s)\).  
\begin{align*}
D_t W_s^p 
&= 
D_t \Bigl(\bigl(Y_T\,Y_s^{-1}\,\sigma(s, X_s)\bigr)^p\Bigr).
\end{align*}

\begin{itemize}
    \item \textbf{Case 1: \(t \leq s\).} 
    In this scenario, a “kick” in the Brownian motion at time \(t\) does 
    affect \(X_s\) (and hence \(Y_s\)). Thus
    \begin{align*}
    D_t \bigl(Y_T\,Y_s^{-1}\,\sigma(s, X_s)\bigr)
    &= 
   D_t Y_T
      \;\cdot\; 
      \bigl(Y_s^{-1}\,\sigma(s, X_s)\bigr)
    \;+\;
    Y_T \;\,\underbrace{D_t\bigl(Y_s^{-1}\,\sigma(s, X_s)\bigr)}_{\text{chain rule}}.
    \end{align*}
Using Lemma \ref{lemma:DtYT} we write $D_t Y_T$ as 
        \begin{align*}
        D_t Y_T &= Z_T\,Y_t^{-1}\,\sigma(t, X_t) - Y_T\,Y_t^{-1}\,Z_t\,Y_t^{-1}\,\sigma(t, X_t) \\
        &\quad + Y_T\,Y_t^{-1}\,\partial_x \sigma(t, X_t)\,Y_t
        \end{align*}
        Note that 
        \begin{align}
        D_t \bigl(Y_s^{-1}\,\sigma(s, X_s)\bigr)
        &= 
        \bigl(D_t Y_s^{-1}\bigr)\,\sigma(s, X_s)
          \;+\;
          Y_s^{-1}\,\bigl(D_t \sigma(s, X_s)\bigr).
          \label{a5}
        \end{align}
        For \( t \leq s \), the Malliavin derivative of the inverse is
        \begin{align*}
        D_t Y_s^{-1} = -Y_s^{-1} \big[ &Z_s Y_t^{-1} \sigma(t, X_t) \\
        &- Y_s Y_t^{-1} Z_t Y_t^{-1} \sigma(t, X_t) \\
        &+ Y_s Y_t^{-1} \partial_x \sigma(t, X_t) Y_t \big] Y_s^{-1}.
        \end{align*}
    This expression accounts for the second variation processes \( Z_s \) and \( Z_t \), as well as the derivative of the diffusion coefficient.
Applying the chain rule
        \begin{align*}
        D_t \sigma(s, X_s)
        \;=\;
        \partial_x \sigma(s, X_s)\,\bigl(Y_s\,Y_t^{-1}\,\sigma(t, X_t)\bigr).
        \end{align*}
        Substituting into \eqref{a5} yields
        \begin{align*}
        D_t \bigl(Y_s^{-1}\,\sigma(s, X_s)\bigr) &= - Y_s^{-1} \Bigl[ Z_s Y_t^{-1} \sigma(t, X_t) \\
        &\quad - Y_s Y_t^{-1} Z_t Y_t^{-1} \sigma(t, X_t) \\
        &\quad + Y_s Y_t^{-1} \partial_x \sigma(t, X_t) Y_t \Bigr] Y_s^{-1} \sigma(s, X_s) \\
        &\quad + Y_s^{-1} \partial_x \sigma(s, X_s) \bigl(Y_s Y_t^{-1} \sigma(t, X_t)\bigr).
        \end{align*}
\end{itemize}
Thus, for \(t \leq s\)
\begin{align*}
\begin{aligned}
D_t W_s^p =
 \Biggl[&\Bigl( Z_T\,Y_t^{-1}\,\sigma(t, X_t)
-\, Y_T\,Y_t^{-1}\,Z_t\,Y_t^{-1}\,\sigma(t, X_t) \\
&\quad +\, Y_T\,Y_t^{-1}\,\partial_x \sigma(t, X_t)\,Y_t \Bigr)\,Y_s^{-1}\,\sigma(s, X_s) \\
&\quad +\, Y_T \Biggl( 
-\, Y_s^{-1}\,\Bigl[ Z_s Y_t^{-1}\,\sigma(t, X_t)
-\, Y_s Y_t^{-1}\,Z_t Y_t^{-1}\,\sigma(t, X_t) \\
&\quad +\, Y_s Y_t^{-1}\,\partial_x \sigma(t, X_t) Y_t \Bigr]\,Y_s^{-1}\,\sigma(s, X_s) \\
&\quad +\, Y_s^{-1}\,\partial_x \sigma(s, X_s) \Bigl( Y_s Y_t^{-1}\,\sigma(t, X_t) \Bigr)
\Biggr)\Biggr]^p.
\end{aligned}
\end{align*}
We place the entire sum inside brackets \(\,[\dots]^p\) because we are taking 
the \(p\)-th component of the resulting vector.
\begin{itemize}
    \item \textbf{Case 2: \(t > s\).} 
    In this case, a Brownian perturbation at time \(t\) does \emph{not} affect \(X_s\) 
    (nor \(Y_s\)) because \(s<t\).  
    Hence
    \begin{align*}
    D_t \bigl(Y_s^{-1}\,\sigma(s, X_s)\bigr)
    &= 
    0,
    \end{align*}
    and the only contribution is from \(D_t Y_T\). 
    Therefore,
    \begin{multline*}
    D_t W_s^p
    = \Biggl[
      \Bigl(
        Z_T\,Y_t^{-1}\,\sigma(t, X_t) \\
        -\, Y_T\,Y_t^{-1}\,Z_t\,Y_t^{-1}\,\sigma(t, X_t)
        +\, Y_T\,Y_t^{-1}\,\partial_x \sigma(t, X_t)\,Y_t
      \Bigr)\,Y_s^{-1}\,\sigma(s, X_s)
    \Biggr]^p.
    \end{multline*}
\end{itemize}
Next, we recall that 
\[\displaystyle \gamma_{X_T}^{p,q} = \int_0^T [Y_T Y_s^{-1} \sigma(s, X_s)]^p \,\cdot\,[Y_T Y_s^{-1} \sigma(s, X_s)]^q\,ds.\]
Recall that $\gamma_{X_T}^{p,q} = \int_0^T [Y_T Y_s^{-1} \sigma(s, X_s)]^p \cdot [Y_T Y_s^{-1} \sigma(s, X_s)]^q\,ds$. Taking the Malliavin derivative and using the product rule
\begin{align}
D_t \gamma_{X_T}^{p,q} &= \int_0^T \Big[ D_t W_s^p \cdot W_s^q + W_s^p \cdot D_t W_s^q \Big]\,ds,
\label{eq:Dt_gamma_compact}
\end{align}
where $W_s^p = (Y_T Y_s^{-1} \sigma(s, X_s))^p$ and we computed $D_t W_s^p$ in Cases 1 and 2 above.
To evaluate \eqref{eq:Dt_gamma_compact}, we split the integration region into $[0,t]$ and $[t,T]$ to reflect the piecewise definitions of $D_t W_s^p$
\begin{itemize}
\item For $s \in [0,t]$: Use the Case 1 expression for $D_t W_s^p$ and $D_t W_s^q$
\item For $s \in [t,T]$: Use the Case 2 expression for $D_t W_s^p$ and $D_t W_s^q$
\end{itemize}
The resulting expression for $D_t \gamma_{X_T}^{p,q}$ involves the second variation processes $Z_t, Z_s, Z_T$ and partial derivatives of the diffusion coefficients, as detailed in the computations above.
We now handle the correction term 
\(\displaystyle \sum_{j=1}^m \partial_j u_t(F_k) \,\cdot\, D_t F_k^j\) 
that appears in the Skorokhod integral decomposition
$\delta(u_k) = \left. \int_0^T u_t(x) \cdot dB_t \right|_{x=F_k} \;-\;\sum_j \int_0^T \partial_j u_t(F_k)\,\cdot\,D_tF_k^j\,dt$.
Hereafter we show how each step follows from the chain rule in 
Malliavin calculus, the use of 
\(\displaystyle D_t (\gamma_{X_T}^{-1}) = -\,\gamma_{X_T}^{-1}\,(D_t \gamma_{X_T})\,\gamma_{X_T}^{-1},\) 
and the expression for 
\(\displaystyle D_t \gamma_{X_T}^{p,q}\). 
We also show how to integrate the resulting expression over \(t\). 
Let us begin by recalling the general formula for the Skorokhod integral
\begin{align*}
\delta(u_k) 
&= \left. \int_0^T u_t(x)\cdot dB_t \right|_{x=F_k} - \int_0^T \sum_{j=1}^m \partial_j u_t(F_k) \,\cdot\, D_t \bigl(F_k^j\bigr) \,dt.
\end{align*}
The term 
\(\displaystyle \sum_{j=1}^m \partial_j u_t(F_k) \,\cdot\,D_t F_k^j\)
is often called the ``correction term.''  
We have already found  that 
\(\,\partial_j u_t(F_k) = \bigl[Y_t^{-1}\,\sigma(t, X_t)\bigr]_j,\) 
and 
\(\,F_k^j = e_j^\top Y_T^\top \gamma_{X_T}^{-1} e_k.\)  
\\Since 
\(\,D_t(\gamma_{X_T}^{-1}) = -\,\gamma_{X_T}^{-1}\,(D_t\gamma_{X_T})\,\gamma_{X_T}^{-1},\) 
we obtain
\[
D_t F_k^j = D_t(e_j^\top Y_T^\top \gamma_{X_T}^{-1} e_k) 
= e_j^\top (D_t Y_T^\top) \gamma_{X_T}^{-1} e_k - e_j^\top Y_T^\top \gamma_{X_T}^{-1} (D_t \gamma_{X_T}) \gamma_{X_T}^{-1} e_k.
\]
Hence,
\begin{align}
\label{eq:correction-term-raw}
\sum_{j=1}^m \partial_j u_t(F_k) \cdot D_t F_k^j 
&= \sum_{j=1}^m \bigl[ Y_t^{-1} \sigma(t, X_t) \bigr]_j 
   \left( e_j^\top (D_t Y_T^\top) \gamma_{X_T}^{-1} e_k - e_j^\top Y_T^\top \gamma_{X_T}^{-1} (D_t \gamma_{X_T}) \gamma_{X_T}^{-1} e_k \right).
\end{align}
Now, recall that \(\displaystyle D_t \gamma_{X_T}^{p,q}\) splits into integrals over \([0,t]\) and \([t,T]\), 
and contains contributions from terms like 
\(Z_T\,Y_t^{-1}\,\sigma(t,X_t)\), 
\(\,Y_T\,Y_t^{-1}\,Z_t\,Y_t^{-1}\,\sigma(t,X_t)\), 
\(\,Y_T\,Y_t^{-1}\,\partial_x \sigma(t,X_t)\,Y_t,\) 
and so forth  (as detailed in Cases 1 and 2). 
To complete the calculation of the correction term, we substitute \eqref{eq:Dt_gamma_compact} into \eqref{eq:correction-term-raw}, reorganise by splitting the integration regions, and integrate from $t=0$ to $t=T$. This yields,
\begin{align*}
\int_0^T \sum_{j=1}^m \partial_j u_t(F_k) \,\cdot\, D_t F_k^j \, dt &= \int_0^T \sum_{j=1}^m \bigl[ Y_t^{-1} \sigma(t, X_t) \bigr]_j \cdot A_{jk}(t) \, dt \\
&\quad - \int_0^T \sum_{j=1}^m \bigl[ Y_t^{-1} \sigma(t, X_t) \bigr]_j \cdot B_{jk}(t) \, dt \\
&\quad - \int_0^T \sum_{j=1}^m \bigl[ Y_t^{-1} \sigma(t, X_t) \bigr]_j \cdot C_{jk}(t) \, dt
\end{align*}
where 
\begin{align*}
A_{jk}(t) &= e_j^\top \bigg[ \sigma(t, X_t)^\top (Y_t^{-1})^\top Z_T^\top - \sigma(t, X_t)^\top (Y_t^{-1})^\top Z_t^\top (Y_t^{-1})^\top Y_T^\top \\
&\quad\quad\quad + \left( Y_T Y_t^{-1} \partial_x \sigma(t, X_t) Y_t \right)^\top \bigg] \gamma_{X_T}^{-1} e_k \\[0.5em]
B_{jk}(t) &= e_j^\top Y_T^\top \gamma_{X_T}^{-1} \cdot \bigg[ \int_0^t I_1(t,s) \, ds + \int_0^t I_2(t,s) \, ds \bigg] \gamma_{X_T}^{-1} e_k \\[0.5em]
C_{jk}(t) &= e_j^\top Y_T^\top \gamma_{X_T}^{-1} \cdot \bigg[ \int_t^T I_3(t,s) \, ds + \int_t^T I_4(t,s) \, ds \bigg] \gamma_{X_T}^{-1} e_k
\end{align*}
with the integrands having components
\begin{align*}
I^{p,q}_1(t,s) &= \bigg[ \Omega(t) Y_s^{-1}\sigma(s, X_s) \bigg]^p \cdot [Y_T Y_s^{-1} \sigma(s, X_s)]^q \\[0.5em]
I^{p,q}_2(t,s) &= [Y_T Y_s^{-1} \sigma(s, X_s)]^p \cdot \bigg[ \Omega(t) Y_s^{-1}\sigma(s, X_s) \bigg]^q \\[0.5em]
I^{p,q}_3(t,s) &= \bigg[ \Omega(t) Y_s^{-1}\sigma(s, X_s) + Y_T \Theta(t,s) \bigg]^p \cdot [Y_T Y_s^{-1} \sigma(s, X_s)]^q \\[0.5em]
I^{p,q}_4(t,s) &= [Y_T Y_s^{-1} \sigma(s, X_s)]^p \cdot \bigg[ \Omega(t) Y_s^{-1}\sigma(s, X_s) + Y_T \Theta(t,s) \bigg]^q,
\end{align*}
where
\begin{align*}
\Omega(t) &= Z_T Y_t^{-1} \sigma(t, X_t) - Y_T Y_t^{-1} Z_t Y_t^{-1} \sigma(t, X_t) + Y_T Y_t^{-1} \partial_x \sigma(t, X_t) Y_t \\[0.5em]
\Theta(t,s) &= -Y_s^{-1} \bigg[ Z_s Y_t^{-1} \sigma(t, X_t) - Y_s Y_t^{-1} Z_t Y_t^{-1} \sigma(t, X_t) \\
&\quad + Y_s Y_t^{-1} \partial_x \sigma(t, X_t) Y_t \bigg] Y_s^{-1} \sigma(s, X_s) \\
&\quad + Y_s^{-1} \partial_x \sigma(s, X_s) \bigg(Y_s Y_t^{-1} \sigma(t, X_t) \bigg).
\end{align*}
These expansions show how the correction term 
\(\sum_{j=1}^m \partial_j u_t(F_k) \cdot D_t F_k^j\) 
expands in terms of 
\(\,D_t \gamma_{X_T}^{p,q}\). 
 %
%
This yields the final formula for the  Skorokhod integral is given in \eqref{score1}.

\noindent

\end{proof}

\section{State-independent diffusion processes}
\label{sec:simplifiedformula}

In this section, we derive a simplified expression for the Skorokhod integral $\delta(u_k)$ for SDEs with state-independent diffusion coefficients. 
Such SDEs are of interest in diffusion generative modelling because the noise amplitude remains identical across all sample paths, and can be chosen a priori via an appropriate noise scheduler. From a numerical standpoint, stochastic differential equations with state-independent diffusion coefficients exhibit reduced stiffness compared to those with multiplicative noise. This allows for larger time steps during sampling and offers direct control over the evolution of the marginal variance. This simplifies the computation of likelihood weighting and score-matching losses. As a result, most practical architectures for diffusion-based generative modelling are based on state-independent diffusion coefficients, such as the VP (Variance-Preserving), VE (Variance-Exploding), and sub-VP SDEs introduced in \cite{song2021scorebased}. Hence, consider SDEs of the form
\[
  dX_t \;=\; b(t,X_t)\,dt \;+\; \sigma(t)\,dB_t,
\]
where the diffusion coefficient $\sigma:[0,T]\to\mathbb{R}^{m\times d}$ depends on time~$t$ only. This property will enable the simplifications of $\delta(u_k)$ stated in the forthcoming corollary.
Also, for the linear case with additive (state-independent) noise, the resulting expression reduces to the classical score formulae for linear Gaussian SDEs, [see \cite{mirafzali2025malliavincalculusscorebaseddiffusion}].

\begin{corollary}[Skorokhod integral for state-independent diffusion]
\label{cor:state_independent}
Consider the stochastic differential equation
\[dX_t = b(t, X_t) \, dt + \sigma(t) \, dB_t,\]
where the diffusion coefficient \(\sigma: [0, T] \to \mathbb{R}^{m \times d}\) depends exclusively on time \(t\). Under the assumptions of Theorem~\ref{thm:main}, the Skorokhod integral \(\delta(u_k)\) simplifies to
\begin{align*}
\delta(u_k) &= \left. \int_0^T u_t(x) \cdot dB_t \right|_{x=F_k} \\
&\quad - \int_0^T \sum_{j=1}^m \bigl[ Y_t^{-1} \sigma(t) \bigr]_j \cdot A_{jk}(t) \, dt \\
&\quad + \int_0^T \sum_{j=1}^m \bigl[ Y_t^{-1} \sigma(t) \bigr]_j \cdot B_{jk}(t) \, dt \\
&\quad + \int_0^T \sum_{j=1}^m \bigl[ Y_t^{-1} \sigma(t) \bigr]_j \cdot C_{jk}(t) \, dt
\end{align*}
where
\begin{align*}
A_{jk}(t) &= e_j^\top \bigg[ \sigma(t)^\top (Y_t^{-1})^\top Z_T^\top - \sigma(t)^\top (Y_t^{-1})^\top Z_t^\top (Y_t^{-1})^\top Y_T^\top \bigg] \gamma_{X_T}^{-1} e_k \\[0.5em]
B_{jk}(t) &= e_j^\top Y_T^\top \gamma_{X_T}^{-1} \cdot \bigg[ \int_0^t I_1(t,s) \, ds \bigg] \gamma_{X_T}^{-1} e_k \\[0.5em]
C_{jk}(t) &= e_j^\top Y_T^\top \gamma_{X_T}^{-1} \cdot \bigg[ \int_t^T I_2(t,s) \, ds \bigg] \gamma_{X_T}^{-1} e_k
\end{align*}
with integrands
\begin{align*}
I_1(t,s) &= \bigg[ \Omega(t) Y_s^{-1}\sigma(s) \bigg] W_s^\top + W_s \bigg[ \Omega(t) Y_s^{-1}\sigma(s) \bigg]^\top \\[0.5em]
I_2(t,s) &= \bigg[ \Phi(t,s) \bigg] W_s^\top + W_s \bigg[ \Phi(t,s) \bigg]^\top
\end{align*}
and auxiliary processes
\begin{align*}
\Omega(t) &= Z_T Y_t^{-1} \sigma(t) - Y_T Y_t^{-1} Z_t Y_t^{-1} \sigma(t) \\[0.5em]
\Phi(t,s) &= \Omega(t) Y_s^{-1}\sigma(s) - Y_T Y_s^{-1} \bigg[ Z_s Y_t^{-1} \sigma(t) - Y_s Y_t^{-1} Z_t Y_t^{-1} \sigma(t) \bigg] Y_s^{-1} \sigma(s) \\[0.5em]
W_s &= Y_T Y_s^{-1} \sigma(s)
\end{align*}
where \(u_t(x) = x^\top Y_t^{-1} \sigma(t)\) and \(F_k = Y_T^\top \gamma_{X_T}^{-1} e_k\). The state-independence of \(\sigma(t)\) eliminates all terms involving \(\partial_x \sigma(t, X_t)\), significantly simplifying the expression compared to the general case.
\end{corollary}

\label{sec:summary}

\section*{Acknowledgements}
\noindent\textbf{AI assistance.} We used ChatGPT only for minor editing (typos and wording). No proofs or mathematical content were generated, and no scientific content was altered.

\section*{Disclosure statement}
The authors declare no potential conflict of interest.

\section*{Funding}
Daniele Venturi was supported by the U.S. Department of Energy (DOE) grant DE-SC0024563.


\end{document}